\documentclass{article}

\usepackage{microtype}
\usepackage{graphicx}
\usepackage{subfigure}
\usepackage{booktabs} % for professional tables
\usepackage{bm}
\usepackage{caption}
\usepackage{subcaption}
\usepackage{hyperref}

\usepackage[accepted]{icml2025}

\usepackage{amsmath}
\usepackage{amssymb}
\usepackage{mathtools}
\usepackage{amsthm}

\usepackage[capitalize,noabbrev]{cleveref}

\theoremstyle{plain}
\newtheorem{theorem}{Theorem}[section]

\newtheorem{lemma}[theorem]{Lemma}
\newtheorem{corollary}[theorem]{Corollary}
\theoremstyle{definition}
\newtheorem{definition}[theorem]{Definition}

\theoremstyle{remark}

\newtheorem{fact}[theorem]{Fact}

\usepackage[textsize=tiny]{todonotes}

\icmltitlerunning{Connecting Thompson Sampling and UCB: Towards More Efficient Trade-offs Between Privacy and Regret}

\begin{document}

\twocolumn[
\icmltitle{Connecting Thompson Sampling and UCB: Towards More Efficient Trade-offs Between Privacy and Regret}



\icmlsetsymbol{equal}{*}

\begin{icmlauthorlist}
\icmlauthor{Bingshan Hu}{yyy}
\icmlauthor{Zhiming Huang}{xxx}
\icmlauthor{Tianyue H. Zhang}{comp,mila}
\icmlauthor{Mathias Lécuyer}{yyy}
\icmlauthor{Nidhi Hegde}{sch,amii}
% \\

% \icmlauthor{1 UBC}

\end{icmlauthorlist}

\icmlaffiliation{yyy}{Department of Computer Science, University of British Columbia, Canada.}
\icmlaffiliation{xxx}{Department of Computer Science, University of Victoria, Canada.}
\icmlaffiliation{comp}{Université de Montréal, Canada.}
\icmlaffiliation{mila}{Mila -- Quebec AI Institute, Canada.}
\icmlaffiliation{sch}{Department of Computing Science, University of Alberta, Canada.}
\icmlaffiliation{amii}{Alberta Machine Intelligence Instribute (Amii), Canada}

\icmlcorrespondingauthor{Bingshan Hu}{bingshanhu3@gmail.com}

\icmlkeywords{Machine Learning, ICML}

\vskip 0.3in
]



\printAffiliationsAndNotice{}  % leave blank if no need to mention equal contribution

\begin{abstract}

% old version
 %We address differentially private stochastic bandit problems from the angles of exploring the deep connections among Thompson Sampling with Gaussian priors, Gaussian mechanisms, and  Gaussian differential privacy (GDP).  We propose DP-TS-UCB, a novel parametrized private bandit algorithm that enables to trade off privacy and regret. DP-TS-UCB satisfies $ \tilde{O} \left(T^{0.25(1-\alpha)}\right)$-GDP and enjoys an $O \left(K\ln^{\alpha+1}(T)/\Delta \right)$ regret bound, where $\alpha \in [0,1]$ controls the trade-off between privacy and regret. Theoretically, our DP-TS-UCB relies on anti-concentration bounds of Gaussian distributions and links exploration mechanisms in Thompson Sampling-based algorithms and Upper Confidence Bound-based algorithms, which may be of independent interest.

We address differentially private stochastic bandit problems by leveraging  Thompson Sampling with Gaussian priors  and Gaussian differential privacy (GDP). 
 We propose DP-TS-UCB, a novel parametrized private  algorithm that enables trading off privacy and regret. DP-TS-UCB satisfies $ \tilde{O} \left(T^{0.25(1-\alpha)}\right)$-GDP and achieves
 $O \left(K\ln^{\alpha+1}(T)/\Delta \right)$ regret bounds, where $K$ is the number of arms, $ \Delta$ is the sub-optimality gap, $T$ is the learning horizon, and
 $\alpha \in [0,1]$ controls the trade-off between privacy and regret. Theoretically, 
 DP-TS-UCB relies on anti-concentration bounds for the Gaussian distributions, linking the exploration mechanisms of Thompson Sampling and Upper Confidence Bound, which may be of independent research interest.

\end{abstract}

\section{Introduction}\label{sec: intro}

This paper studies differentially private
stochastic  bandit problems previously studied in~\citet{mishra2015nearly,hu2021near,hu2022near,azize2022privacy,ou2024thompsonsamplingdifferentiallyprivate}.  
In a classical stochastic bandit problem, we have a fixed arm set $[K]$. Each arm  $i $ is associated with a fixed but unknown reward distribution $p_i$ with mean reward $\mu_i$. 
In each round, a learning agent pulls an arm and obtains a random reward that is distributed according to the reward distribution associated with the pulled arm. The goal of the learning agent is to pull arms sequentially to accumulate as much reward as possible over a finite number of $T$ rounds.  Since 
%the learning agent needs to run some algorithm to decide which arms to pull and 
the pulled arm in each round may not always be the optimal one, \emph{regret}, defined as the expected cumulative loss between the highest mean reward and the earned mean reward, is used to measure the performance of the algorithm used by the learning agent to make decisions.

Low-regret bandit algorithms should leverage past information to inform future decisions, as previous observations reveal which arms have the potential to yield higher rewards.
%made in round $t$ depending on all the information gathered in previous rounds.}
%\textcolor{red}{NH:  something missing in this last sentence?  What are you trying to say?}
However, due to privacy concerns, the learning agent may not be allowed to directly use past information %\textcolor{pink}{?} 
 to make decisions. For example,  a hospital collects health data from patients participating in clinical trials over time to learn the side effects of some newly developed treatments.
To comply with privacy regulations, the hospital is required to publish scientific findings in a differentially private manner, as the sequentially collected data from patients carries sensitive information from individuals. 
The framework of differential privacy (DP) \citep{dwork2014algorithmic}
is widely accepted to preserve the privacy of individuals whose data have been used for data analysis. 
 Differentially private learning algorithms 
 %ensure that even if an adversary sees the output of the learning, most likely it cannot infer anything useful about any individual. 
bound the privacy loss, the amount of information that
 an external observer can infer about individuals. %One of the approaches to protect privacy is to add noise to  the data.

 % the same thing if one element in the dataset is changed or missing.

DP is commonly achieved by %can be achieved \textcolor{pink}{?} by 
adding noise to summary statistics computed based on the collected data. Therefore, 
 to solve a private bandit problem, the learning agent has to navigate two trade-offs. %the learning agent faces two trade-offs. 
 %\textcolor{pink}{%Therefore, the private bandit problem presents an additional challenge, as } 
  The first one is the \emph{fundamental trade-off between exploitation and exploration} due to bandit feedback:  %when making decisions in each round. %when deciding which action to take and what to collect: 
in each round, the learning agent can only focus on either exploitation (pulling arms seemingly promising to attain reward) or exploration (pulling arms helpful to learn the unknown mean rewards and reduce uncertainty). 
 The second one is the \emph{trade-off between privacy and regret} due to the DP noise: 
adding more noise enhances privacy, but it reduces data estimation accuracy and weakens regret guarantees.

% the fundamental challenge is how to make wise sequential decisions based on noisy data. We have two consider the following two aspects simultaneously. {\color{blue}Trade off between exploitation vs exploration and Trade off between accuracy and privacy}

There are two main strategies to design (non-private) stochastic bandit algorithms that efficiently balance exploitation and exploration: Upper Confidence Bound (UCB) \citep{auer2002finite} and Thompson Sampling \citep{agrawalnear,kaufmann2012thompson}.
Both enjoy good theoretical regret guarantees and competitive empirical performance.
 %-vs- in bandit problem (1) UCB and TS can achieve a good exploitation-vs-exploration trade-off. 
 UCB is inspired by the principle of optimism in the face of uncertainty, adding deterministic bonus terms to the empirical estimates based on their uncertainty to achieve exploration. %and the upper bounds of the  confidence intervals are used in the learning. 
 % The exploration is achieved by adding bonus terms to the empirical estimates. 
Thompson Sampling is inspired by Bayesian learning, %relying on the idea of 
using the idea of sampling mean reward models from posterior distributions (e.g., Gaussian distributions) that model the unknown mean rewards of each arm. 
%The learning agent samples  random models from these distributions to make decisions.  
The procedure of sampling mean reward models can be viewed as adding random bonus terms to the empirical estimates. %\textcolor{pink}{The model sampling process can be seen as adding random bonus terms to empirical estimates of each arm's mean reward to encourage exploration.} {\color{blue}(The key is to discuss about the exploration mechanisms behind TS and UCB.)}% and the variances of the posterior distributions control the spread of the added bonus.
%by Bayes...are confirmed to be efficient to tackle the exploitation-vs-exploration dilemma.

The design of the existing private stochastic bandit algorithms \citep{sajed2019optimal,hu2021near,azize2022privacy,hu2022near} follows the framework of  adding calibrated noise  to the empirical estimates first to achieve privacy. Then, the learning agent  makes decisions based on noisy estimates, which can be viewed as post-processing that preserves DP guarantees.
Since both Thompson Sampling and DP algorithms rely on   adding noise to the empirical estimates, it is natural to wonder whether the existing Thompson Sampling-based algorithms offer some level of privacy at no additional cost, without compromising any regret guarantees.

Very recently, 
\citet{ou2024thompsonsamplingdifferentiallyprivate} show that  Thompson Sampling with Gaussian priors \citep{agrawalnear} (we rename it as TS-Gaussian in this work) without any modification is indeed DP by leveraging Gaussian privacy mechanism (adding Gaussian noise to the collected data \citep{dwork2014algorithmic}) and the notion of Gaussian differential privacy (GDP) \citep{dong2022gaussian}. They show that TS-Gaussian is $O(\sqrt{T})$-GDP.
However, this privacy guarantee is not tight due to the fact that \emph{TS-Gaussian has to sample a mean reward model from a data-dependent Gaussian distribution for each arm in each round to achieve  exploration}. 
Each sampled Gaussian mean reward model implies the injection of some Gaussian noise into the empirical estimate, and sampling in total $T$ Gaussian mean reward models for each arm 
 provides a privacy guarantee in the order of $\sqrt{T}$.

In this paper, we propose  a novel private bandit algorithm, DP-TS-UCB (presented in Algorithm~\ref{alg:private}), which does not require sampling a Gaussian mean reward model in each round, 
  % and   doing composition over $T$ rounds, 
and is hence % and 
   more efficient at trading off privacy and regret.  
Theoretically, DP-TS-UCB  uncovers the connection between exploration mechanisms in TS-Gaussian and UCB1 \citep{auer2002finite}, which may be of independent interest. %{\color{red}Maybe add some explanation...}

%anti-concentration bounds of Gaussian distributions, connect UCB and TS.

  Our proposed algorithm builds upon the insight that, for each arm $i$, the Gaussian distribution that models the mean reward of arm $i$ can only change when arm $i$ is pulled, as a new pull of arm $i$ indicates the arrival of new data associated with arm $i$. In other words, the Gaussian distribution stays the same in all rounds between two consecutive pulls of arm $i$. 
 % {\color{red}We can partition all rounds between two consecutive pulls of that arm into two phases with $\phi$ rounds in Phase I and in each round in Phase I, drawing a Gaussian a model is allowed. That is also to say, }
%{\color{red}We propose a carefully designed sampling budget that reduces unnecessary noise injection, enhancing privacy guarantees without compromising regret performance.}
 Based on this insight, to avoid unnecessary Gaussian sampling, which increases privacy loss, DP-TS-UCB \textbf{sets a budget 
 $\phi$ for
 the  number of  Gaussian mean reward models that are allowed to draw from a   Gaussian distribution.}  %For a fixed arm, 
Among all the rounds between two consecutive pulls of arm $i$, DP-TS-UCB can only draw a Gaussian mean reward model in each of the first $\phi$ rounds. % from this Gaussian distribution. 
    % between two consecutive pulls of that arm, DP-TS-UCB first 
   If arm $i$ is  still not pulled after  $\phi$ rounds, DP-TS-UCB reuses the highest model value  among the previously sampled $\phi$ Gaussian mean reward models in the remaining rounds until arm $i$ is pulled again. %{\color{red}Since the number of rounds between two consecutive pulls of arm \(i\), thereby controlling the injected Gaussian noise.} 
Figure~\ref{fig:enter-label} presents a concrete example of how DP-TS-UCB works. %Reusing the Gaussian model enables to sample $\phi$ Gaussian models at most no matter how many rounds between two consecutive pulls of arm $i$. Therefore, the injected Gaussian noise can be controlled. 

   \begin{figure}[!ht]
    \centering
\includegraphics[trim=0 130 0 210,clip,width=0.5\textwidth]{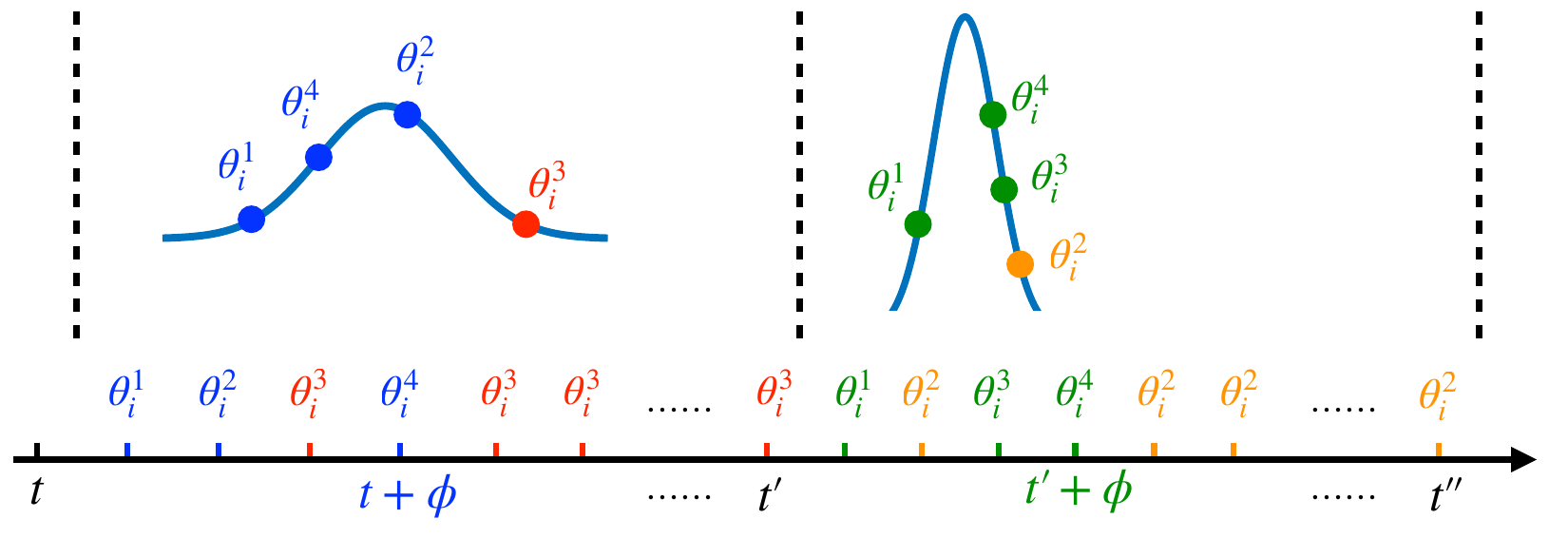}
    \caption{{\textbf{Cap the number of mean reward models sampled from a Gaussian distribution.}} Assume arm $i$ is pulled in rounds $t$, $t'$ and $t''$, and $\phi =4$. In each of the rounds $t+1, \dotsc, t+h, \dotsc, t+\phi$, DP-TS-UCB samples a Gaussian mean reward model {$\theta_i^{h} $} and uses it in the learning for arm $i$. In each of the  rounds $t+\phi+1, t+\phi+2, \dotsc, t'$, DP-TS-UCB reuses the highest model value  ${\theta_i^{3}} = \mathop{\max}_{h \in [\phi]} {\theta_i^{h}}$ among the previously sampled $\phi$ mean reward models in the learning for arm $i$. {Once a new Gaussian distribution is available (the Gaussian distribution located on the right side), DP-TS-UCB is allowed to draw $\phi$ Gaussian mean reward models again in each of the rounds $t'+1, t'+2, \dotsc, t'+\phi $.}} 
    \label{fig:enter-label}
\end{figure}
\begin{figure}[ht]
    \centering
    \includegraphics[trim=0 0 0 0,clip,width=0.5\textwidth]{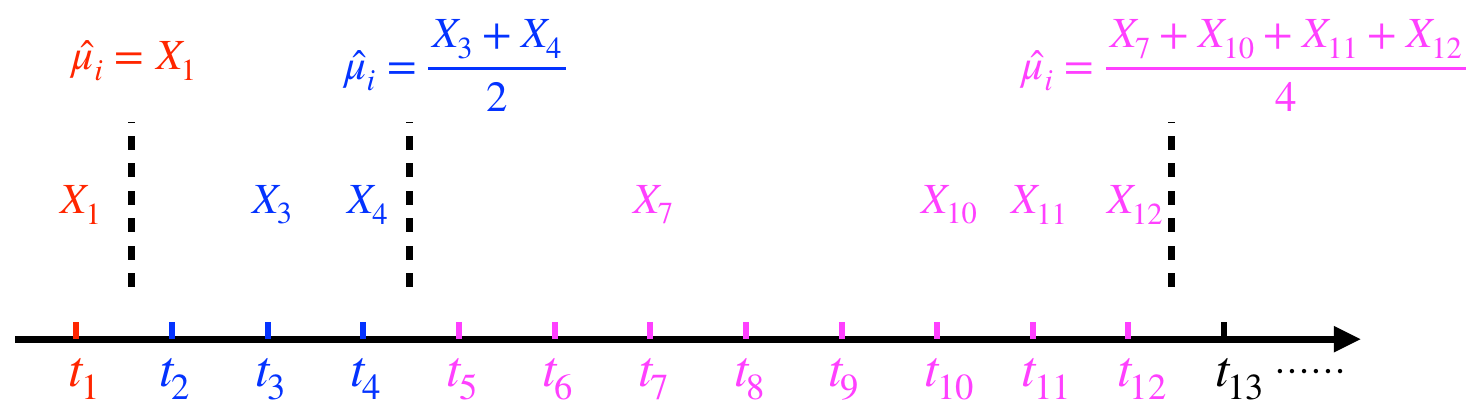}
    \caption{\textbf{Arm-specific epoch structure.} {The dashed lines partition rounds from $t_1$ to $t_{12}$ into three epochs.} 
    Assume arm $i$ is pulled in round $t_1$, then we compute its empirical mean as $\hat{\mu}_i = X_1$ at the end of round $t_1$ and arm $i$'s first epoch ends in round $t_1$. If arm $i$ is pulled in rounds  $t_3, t_4$ again, then we compute its empirical mean as $\hat{\mu}_i = (X_3+X_4)/2$ at the end of round $t_4$ and arm $i$'s second epoch ends in round $t_4$. It is important to  note that arm $i$'s empirical mean will not be updated at the end of round $t_3$ even though it is pulled in round $t_3$.} %{\color{magenta}$t_7, t_{10}, t_{11},t_{12}$.}
    \label{fig:arm_specific_epoch}
\end{figure}

%Note that only new observations can make the Gaussian distribution updated.
 %\begin{enumerate}
    %  \item  
     % \item  Once we are confident some arm is sub-optimal, we do not need to sample new  models for that arm.
       % \end{enumerate}
  % (1) only new observations can change the Gaussian distribution associated with a given arm; (2) With these insights in hand, 
  %  DP-TS-UCB \emph{only samples models when noticing exploration is necessary for some arm rather than sampling a model in each round for each arm}. 
%We can view DP-TS-UCB as a  two-phase algorithm with a mandatory Thompson Sampling phase (Phase I: sampling random model) and an optional \emph{Gaussian sample re-usage} phase. 
%Since we set a cap the number of random models is allowed,  DP-TS-UCB controls the number of rounds when noise is added.

 % TS, tightening privacy cost... 

 %We only sample models when we notice exploration is necessary for some arm.

 % For (1), we can cap the number of Gaussian samples that can be drawn from a given Gaussian distribution. Phase I. TS

 % For (2), after we consume all the sampling budget, we can reuse the best Gaussian sample to guide the learning. Phase II.  UCB

 % The theoretical reasons are: (1)using anti-concentration properties of the chosen data-dependent distributions and connecting TS and UCB.  (3)  

To have a tight privacy guarantee,  %privacy composition %\textcolor{pink}{composition of what? Privacy leaks due to large scale of observations? or composition of errors due to added noise?} over $T$ rounds, 
in addition to capping the number of Gaussian mean reward models, %that are allowed to  draw, %from a given Gaussian distribution, 
        we also need to limit %the number of Gaussian samples that can be drawn from a given Gaussian distribution. To limit
      the number of times that a revealed observation can be used when computing empirical estimates. Similar to \citet{sajed2019optimal,hu2021near,azize2022privacy,hu2022near}, we use an \textbf{arm-specific epoch structure} to process the revealed observations. As   already discussed in these works, using this structure is the key to designing good private online learning algorithms. The key idea of this structure is to update the empirical estimate using  the most recent $2^{r}$ observations, where $r \ge 0$. Figure~\ref{fig:arm_specific_epoch} illustrates this structure for the first three epochs. 
      
      %As discussed, using this structure has been confirmed to be the key  to develop good private online learning algorithms. 

\textbf{Preview of results.} DP-TS-UCB uses an input parameter $\alpha \in [0,1]$ to control the trade-off between privacy and regret, and the choice of $\phi = O ( T^{0.5(1-\alpha)}  \ln^{0.5(3-\alpha)}(T))$  depends on both $\alpha$ and the learning horizon $T$. %Setting $\phi$ in this way is to achieve sufficient exploration 
 %, which .
Our technical Lemma~\ref{lemma boost} shows that  this choice of $\phi$ ensures sufficient exploration, that is, giving enough optimism, for the  rounds  when  sampling new Gaussian  mean reward models is not allowed. %As will be shown in Lemma~\ref{lemma boost} this choice of $\phi$  connects exploration mechanisms in TS-Gaussian and UCB1. % $\phi$ . %(it seems to depend on the variance of the Gaussian.)
%Let denote the number of Gaussian samples that are allowed to draw from a given Gaussian distribution, where $\alpha \in [0,1]$ controls the value of $\phi$. The choice of $\alpha$ . {\color{red}.} The privacy guarantee of {\color{red}(this part needs to come back)
DP-TS-UCB is
$\tilde{O} (   T^{0.25(1-\alpha)}   )$-GDP (Theorem~\ref{thm:dp-dp})
and 
 % and a %regret bound
%The regret bound of DP-TS-UCB is % instead of $\sqrt{T}$. Also, DP-TS-UCB enjoys a 
 achieves  $\sum_{i: \Delta_i >0} O ( \ln(\phi T \Delta_i^2)  \ln^{\alpha}(T)/\Delta_i )$ regret bounds (Theorem~\ref{thm: regret}), where $\Delta_i$ is the mean reward gap between the optimal arm and a sub-optimal arm $i$. For the case where $\alpha = 0$, DP-TS-UCB enjoys the optimal $\sum_{i: \Delta_i >0} O ( \ln(\phi T \Delta_i^2)  /\Delta_i )$ regret bounds and satisfies $\tilde{O} (   T^{0.25}   )$-GDP, which improves the previous $O(\sqrt{T} )$-GDP guarantee significantly. For the case where $\alpha =1$, DP-TS-UCB satisfies constant $\tilde{O} (  1   )$-GDP and achieves $\sum_{i: \Delta_i >0} O ( \ln(\phi T \Delta_i^2)  \ln(T)/\Delta_i)$ regret bounds. %Table~\ref{tab:regret_privacy} summarizes the theoretical results of the state-of-the-art private algorithms in detail.

\section{Learning Problem}
%{\color{red}Add the structure of this section.} 

In this section, we first present the learning problem of stochastic bandits and then we provide key knowledge related to differentially private online learning.

\subsection{Stochastic Bandits}
In a classical stochastic bandit problem, we have a fixed arm set $[K]$ of size $K$, and each arm $i \in [K]$ is associated with a fixed but unknown reward distribution $p_i$ with $[0,1]$ support. Let $\mu_i$ denote the mean of reward distribution $p_i$. Without loss of generality, we assume that the first arm is the unique optimal arm, i.e., $\mu_1 > \mu_i$ for all $i \ne 1$. Let $\Delta_i:= \mu_1 - \mu_i$ denote the mean reward gap. 
The learning protocol is in each round $t$, a reward vector $X_{t} := \left(X_1(t), X_2(t), \dotsc, X_K(t) \right)$ is generated, where each $X_{i}(t) \sim p_{i}$. Simultaneously,
the learning agent pulls an arm $i_t \in [K]$. At the end of the round, the learning agent receives a reward $X_{i_t}(t)$.
%Given a finite time horizon  $T$,
The goal of the learning agent is to pull arms sequentially to maximize the cumulative reward over $T$ rounds, or equivalently, minimize the \emph{(pseudo)-regret}, defined as
\begin{equation}
\label{regret def bandit}
\begin{array}{lll}
\mathcal{R}(T) &= & 
T \cdot \mu_1 - \mathbb{E} \left[\sum\limits_{t=1}^{T} \mu_{i_t}  \right] \quad,% &= &\sum\limits_{i \in \mathcal{A}: \Delta_i >0} \sum\limits_{t=1}^{T}\mathbb{E} \left[ \bm{1} \left\{i_t=i \right\}\right] \Delta_i,
\end{array}
\end{equation}
 where the expectation is taken over the pulled arm $i_t$.
The regret measures the expected cumulative mean reward loss between always pulling the optimal arm and the learning agent's actual pulled arms.

\subsection{Differential Privacy}
 Our DP  definition in the context of online learning 
  follows the one used in \citet{dwork2014algorithmic,sajed2019optimal, hu2021near,
hu2022near,azize2022privacy,ou2024thompsonsamplingdifferentiallyprivate}.  Let $X_{1:t} := \left(X_1, X_2, \dotsc, X_t \right)$ collect all the reward vectors up to round $t$.  Let $X'_{1:t}$ be a neighbouring sequence of $X_{1:t}$ which differs in at most one reward vector, say, in some round $\tau \le t$.
\begin{definition}[DP in  online learning]
An online learning algorithm  $\mathcal{A}$ is $(\varepsilon, \delta)$-DP if for any two neighbouring reward  sequences $X_{1:T}$ and $X'_{1:T}$, for any decision set $\mathcal{D}_{1:t} \subseteq [K]^t$, we have  $\mathbb{P} \left\{\mathcal{A}({X}_{1:t}) \in \mathcal{D}_{1:t}  \right\} \le e^{\varepsilon} \cdot \mathbb{P} \left\{\mathcal{A}({X}_{1:t}') \in \mathcal{D}_{1:t} \right\} + \delta$ holds for all $t \le T$ simultaneously.
\label{def: classic}
\end{definition}
%\paragraph{Remark.} 

%{\color{red} $(\varepsilon, \delta)$-DP ensures that for all possible neighbouring reward sequences}
% It can be interpreted in terms of 
% % as bounding 
% the max-divergence $D_{\infty}(Q, Q') := \mathop{\max}_{y \in \text{supp}(Q')} \ln \left(\frac{P (Q=y)}{P(Q'=y)} \right)$  between two probability distributions $Q$ and $Q'$, where $Q$ denotes the output distribution (the distribution of the sequentially pulled arms) when working over the true reward sequences ${X}$ and $Q'$ denotes the output distribution when working over  ${X}'$.

  %An $\epsilon$-differentially private algorithm $\mathcal{M} (\cdot)$ ensures that the max-divergence between $Q$ and $Q'$ is at most $\epsilon$, i.e., $\ln \left( \frac{P \left(Q=y\right)}{P \left(Q'=y \right)} \right) \le \epsilon$ for all possible outputs $y$. The value of  $\ln \left( \frac{P \left(Q=y\right)}{P \left(Q'=y \right)} \right)$ quantifies the privacy loss incurred when an adversary witnesses an outcome $y$.

Like \citet{ou2024thompsonsamplingdifferentiallyprivate}, we also perform our analysis using Gaussian differential privacy (GDP)  \citep{dong2022gaussian}, which is well suited to analyzing the composition of Gaussian mechanisms. We then  translate the GDP guarantee to the classical $(\varepsilon, \delta)$-DP guarantee by using the \emph{duality} between GDP and DP (Theorem~\ref{the: duality}).  Indeed, \citet{dong2022gaussian} show that GDP can be viewed as the primal privacy representation with its dual being an infinite collection of $(\varepsilon, \delta)$-DP guarantees.

To introduce  GDP, we first need to define trade-off functions:
%{\color{red}%Here, it seems I need to say something to introduce the f-DP notion. to-be-added... hypotheis testing, type I and type II error, trade off these two types of error}

\begin{definition}
   [Trade-off function \citep{dong2022gaussian}] For any two probability distributions $P$ and $Q$ on the same space, define the trade-off function $T(P,Q): [0,1] \rightarrow [0,1]$ as $T(P,Q)(x) = \mathop{\inf}_{\psi} \left\{\beta_{\psi} : \alpha_{\psi} \le x \right\}$, where $\alpha_{\psi} = \mathbb{E}_{P}[\psi]$, $\beta_{\psi} = 1- \mathbb{E}_Q[\psi]$, and the infimum is taken over all measurable rejection rules $\psi \in [0,1]$. 
\end{definition}

%{\color{red}Say something to introduce GDP.  GDP offers the tightest possible privacy bound of the Gaussian mechanism.} 

  Let $\Phi$ denote the cumulative distribution function (CDF) of the standard normal distribution $\mathcal{N}(0,1)$. To define GDP in the context of online learning, for any $\eta \ge 0$, 
  we let $G_{\eta}(x) :=  T \left(\mathcal{N}(0, 1),  \mathcal{N} \left(\eta,1 \right) \right)(x) = \Phi \left(\Phi^{-1}(1-x)-\eta \right)$ denote the trade-off function of two normal distributions. 

\begin{definition}
[$\eta$-GDP in online learning]  A randomized online learning algorithm $\mathcal{A}$ is $\eta$-GDP if for any two reward vector sequences $X_{1:T}$ and $X'_{1:T}$ differing in at most one vector, we have $T \left(\mathcal{A} \left(X_{1:t} \right),  \mathcal{A} \left(X'_{1:t} \right) \right)(x)  \ge  G_{\eta}(x)$ holds for all $x \in [0,1]$ and $t \le T$ simultaneously.
% \begin{equation}
% \begin{array}{l}
   
%     \end{array}
% \end{equation}
\label{def:GM}
\end{definition}

For easier comparison, we use the following theorem to convert an $\eta$-GDP guarantee to  $(\varepsilon, \delta)$-DP guarantees:
\begin{theorem}[Primal to dual \citep{dong2022gaussian}]
 A randomized algorithm is $\eta$-GDP if and only if it is $(\varepsilon, \delta(\varepsilon))$-DP for all $\varepsilon \ge 0$, where 
 
 $  \delta(\varepsilon) =\Phi \left(-\frac{\varepsilon}{\eta} + \frac{\eta}{2} \right) - e^{\varepsilon}  \Phi \left(-\frac{\varepsilon}{\eta} - \frac{\eta}{2} \right)$\quad.
 \label{the: duality}
\end{theorem}
\textbf{Remark.}
%$\delta(\varepsilon)$ 
Fix any $\varepsilon \ge 0$. We can also view   $\delta(\varepsilon) = \Phi \left(-\frac{\varepsilon}{\eta} + \frac{\eta}{2} \right) - e^{\varepsilon} \Phi \left(-\frac{\varepsilon}{\eta} - \frac{\eta}{2} \right)$ as an increasing function  of $\eta$. This means, for a fixed $\varepsilon$,   the smaller the GDP parameter $\eta$ is, the smaller the $\delta(\varepsilon)$ is after the translation. %{\color{red}add some interpretation for small $\delta(\varepsilon)$.}
     
     %This means, given $\varepsilon$, after the translation, our proposed algorithm has a smaller $\delta(\varepsilon)$ value. %Here, echo $(\varepsilon, \delta)$ stuff, privacy loss.
  %  \end{remark}

% \begin{theorem}
% \label{thm:advanced composition}
% (Advanced Composition, Corollary~3.21 in \cite{dwork2014algorithmic}). Given target privacy parameters $0 < \varepsilon' <1$ and $\delta' >0$, to ensure $(\varepsilon', \phi \delta_0 + \delta') $ cumulative privacy loss over $\phi$ mechanisms, it suffices that each mechanism is $(\varepsilon_0, \delta_0)$-DP, where $\varepsilon_0 = \frac{\varepsilon'}{2\sqrt{2\phi \ln(1/\delta')}}$.
% \end{theorem}

\section{Related Work}
%{\color{blue}In this section, we discuss the most relevant literature. }
%\subsection{Stochastic bandit algorithms}
There is a vast amount of literature on (non-private) stochastic bandit algorithms. 
%We first discuss the exploratio
We split them based on UCB-based versus Thompson Sampling-based, i.e., deterministic versus randomized exploration. Then, we discuss the most relevant algorithms for private stochastic bandits.

UCB-based algorithms~\citep{auer2002finite,audibert2007tuning,garivier2011kl,kaufmann2012bayesian,lattimore2018refining} usually conduct exploration in a deterministic way. The key idea is to construct confidence intervals centred on the empirical estimates. Then, the learning agent makes decisions based on the upper bounds of the confidence intervals. The widths of the confidence intervals control the exploration level. 
Thompson Sampling-based algorithms \citep{agrawalnear,kaufmann2012thompson,bian2022maillard,jin2021mots,jin2022finite,jin2023thompson} conduct exploration in a randomized way. 
The key idea is to use a sequence of well-chosen data-dependent distributions to model each arm's mean reward.  Then, the learning agent makes decisions by sampling random mean reward models from these distributions.   The spread of the data-dependent distributions controls the exploration level.
In addition to the aforementioned algorithms, we also have DMED \citep{honda2010asymptotically}, IMED  \citep{honda2015non},  elimination-style algorithm  \citep{auer2010ucb}, Non-parametric TS \citep{riou2020bandit}, and Generic Dirichlet Sampling \citep{baudry2021optimality}.
All these algorithms enjoy either  $\sum_{i : \Delta_i>0} O ( \ln(T)/\Delta_i )$ or    $\sum_{i : \Delta_i>0} O ( \ln(T)\Delta_i/{\text{KL}\left(\mu_i, \mu_i+ \Delta_i \right) })$ problem-dependent regret bounds, where 
%for the term involving  $T$, where $C \ge 1$ is a universal constant and 
$\text{KL}(a,b)$ denotes the KL-divergence between two Bernoulli distributions with parameters $a,b \in (0,1)$.

%The authors of 
\citet{sajed2019optimal,azize2022privacy, hu2021near} developed optimal $(\varepsilon, 0)$-DP stochastic bandit algorithms by first 
 adding calibrated  Laplace noise to the empirical estimates to ensure $(\varepsilon, 0)$-DP. Then,  eliminating arms  and constructing data-dependent distributions based on noisy estimates can be viewed as post-processing which do not hurt privacy. 
Although \citet{hu2022near} proposed a private Thompson Sampling-based algorithm, it still follows the above recipe without leveraging the inherent randomness present in Thompson Sampling for privacy.

\citet{ou2024thompsonsamplingdifferentiallyprivate} connected  Thompson Sampling with Gaussian priors (we rename it as TS-Gaussian) \citep{agrawalnear} to the Gaussian privacy mechanism \citep{dwork2014algorithmic}  and Gaussian differential privacy \citep{dong2022gaussian}. The idea of TS-Gaussian is to use $\mathcal{N} \left(\hat{\mu}_{i,n_i}, 1/n_i \right)$ to model arm $i$'s mean reward, i.e., the mean of reward distribution $p_i$.  The centre  of the Gaussian distribution $\hat{\mu}_{i,n_i}$ is the empirical average of $n_i$ observations that are i.i.d. according to $p_i$. To decide which arm to pull, for each arm $i$ in each round, the learning agent samples a Gaussian mean reward model $\theta_i \sim \mathcal{N} \left(\hat{\mu}_{i,n_i}, 1/n_i \right)$. The learning agent pulls the arm with the highest mean reward model value.
\citet{ou2024thompsonsamplingdifferentiallyprivate} showed that TS-Gaussian satisfies $\sqrt{0.5T}$-GDP, before translating this GDP guarantee to  $(\varepsilon, \delta)$-DP guarantees with Theorem~\ref{the: duality}. % \citep{dwork2014algorithmic} and  Gaussian differential privacy (GDP) \citep{dong2022gaussian} 
Since there is no modification to  the original algorithm, %They show that 
  %TS-Gaussian %satisfies $\sqrt{T/2}$-GDP and 
the optimal $\sum_{i : \Delta_i>0} O ( \ln(T\Delta_i^2)/\Delta_i )$ problem-dependent regret bounds and the  near-optimal $O (\sqrt{KT\ln (K)} )$ worst-case regret bounds are preserved.
%and . %privacy guarantee and both problem-dependent and problem-independent regret bound} %by the end of the learning is not tight. 
% To improve privacy guarantee, 
\citet{ou2024thompsonsamplingdifferentiallyprivate} also proposed Modified Thompson Sampling with Gaussian priors (we rename it as M-TS-Gaussian), % a modified version of the original TS-Gaussian, 
which enables a privacy and regret trade-off. Compared to TS-Gaussian, the  modifications are pre-pulling each arm $b$ times and scaling the variance of the Gaussian distribution as $c/n_i$. % where $c \ge 1$. %Then, M-TS-Gaussian is able to achieve a trade-off between regret and privacy. 
They proved that M-TS-Gaussian  satisfies $\sqrt{T/(c(b+1))}$-GDP, and achieves  $bK + \sum_{i: \Delta_i >0} O (c \ln(T\Delta_i^2) /\Delta_i)$ problem-dependent regret bounds and $bK + O (c \sqrt{KT \ln K} )$ worst-case regret bounds. Table~\ref{tab:regret_privacy} summarizes the theoretical results of  TS-Gaussian and M-TS-Gaussian with different choices of $b,c$. % %{\color{red}Table~\ref{tab:regret_privacy} summarizes regret bounds and privacy guarantees in detail.}
%by pre-pulling arm arm $b$ times and boosting the variance...add the results of \citet{ou2024thompsonsamplingdifferentiallyprivate}.

The order of $\sqrt{T}$-GDP guarantee from TS-Gaussian and M-TS-Gaussian may not be tight when $T$ is large. 
There are two reasons resulting in this loose privacy guarantee: (1) sampling a Gaussian mean reward model in each round for each arm injects too much noise; (2) repeatedly using  the same observation to compute the empirical estimates creates too much privacy loss. In this work, we propose DP-TS-UCB, a novel private algorithm that does not require sampling a Gaussian mean reward  model in each round for each arm. The intuition is that once we are confident some arm is sub-optimal, we do not need to further explore it. To avoid using the same observation to compute the empirical estimates, %tackle the challenge of rev 
we use the arm-specific epoch structure devised by \citet{hu2021near,azize2022privacy,hu2022near} to process the obtained observations. Using this structure ensures that
 each  observation can only be  used at most once for computing empirical estimates. %

 Regarding  lower bounds with a finite learning horizon $T$ for differentially private stochastic bandits, lower bounds exist  under the classical   $(\varepsilon, \delta)$-DP notion.  \citet{shariff2018differentially} established $\Omega (\sum_{i : \Delta_i>0} \ln(T)/\Delta_i  + K\ln(T)/\varepsilon)$ problem-dependent regret lower bound  and \citet{azize2022privacy} established an $\Omega (\sqrt{KT} + K/\varepsilon)$ minimax regret lower bound for $(\varepsilon, 0)$-DP. \citet{wang2024optimal} established an $\Omega( \sum_{i : \Delta_i>0} \ln(T)/\Delta_i  + \frac{K}{\varepsilon} \ln \frac{ (e^{\varepsilon}-1)T + \delta T}{(e^{\varepsilon}-1) + \delta T}     )$ problem-dependent regret lower bound for $(\varepsilon, \delta)$-DP. In this work, we do not provide any new lower bounds. Our theoretical results are compatible with these established lower bounds.

\begin{algorithm}[tb]
   \caption{DP-TS-UCB}
   \label{alg:private}
\begin{algorithmic}[1]
   \STATE \textbf{Input:} trade-off  parameter $\alpha \in [0,1]$, learning horizon $T$, and budget $\phi = c_0 T^{0.5(1-\alpha)}  \ln^{0.5(3-\alpha)}(T)$. 
    
     \STATE \label{alg:initialization}{\textbf{Initialization:}} (1) pull each arm $i$ once to initialize $n_i$ and $\hat{\mu}_{i, n_i}$, (2) set arm-specific epoch index $r_i \leftarrow 1$ and the number of unprocessed observations  $O_i \leftarrow 0$,  (3) set remaining Gaussian sampling budget $h_i \leftarrow \phi$ and the highest Gaussian mean reward model
${\text{MAX}}_i \leftarrow 0$.
    \FOR{$t = K+1,K+2,\dotsc, T$}
  \FOR {$i \in [K]$}

  \IF { $h_i \ge 1$} 
  
  \STATE \label{alg:phase_one}   
  {Set $\theta_{i}(t)  \leftarrow \theta^{(h_i)}_{i, n_i}$, where 
  $\theta^{(h_i)}_{i, n_i}  \sim \mathcal{N}\left(\hat{\mu}_{i,n_i}, \frac{\ln^{\alpha}(T)}{n_{i}} \right)$} {\color{blue} $\%$Mandatory TS-Gaussian} \label{phase: TS}
   \STATE Set $h_i \leftarrow h_i - 1 $,  ${\text{MAX}}_i \leftarrow \max\{{\text{MAX}}_i, \theta^{(h_i)}_{i, n_i} \}$

%  Append $\theta_{i}(t)$ to $\Phi_i$
\ELSE
\STATE  {Set $\theta_{i}(t) \leftarrow {\text{MAX}}_i$}   {\color{blue} $\% $Optional UCB} \label{alg:phase_two}
  \ENDIF
\ENDFOR
\STATE Pull arm $i_t \in \arg \mathop {\max }_{i \in [K]} \theta_i(t)$, observe $X_{i_t}(t)$, and   set $O_{i_t} \leftarrow O_{i_t} + 1$  

 \IF{$O_{i_t} = 2^{r_{i_t}}$}
 \STATE Compute  $\hat{\mu}_{i_t,n_{i_t}}$, where $n_{i_t}= 2^{r_{i_t}}$
 \STATE Reset $h_{i_t} \leftarrow \phi $, ${\text{MAX}}_{i_t} \leftarrow 0$
 \STATE  Set $r_{i_t} \leftarrow r_{i_t} + 1$ and reset $O_{i_t} \leftarrow 0$.
     \ENDIF 
   \ENDFOR
\end{algorithmic}
\end{algorithm}

%\subsection{DP-TS-UCB}
\section{DP-TS-UCB}
We present DP-TS-UCB and then provide its regret (Theorem~\ref{thm: regret}) and privacy (Theorems~\ref{thm:dp-dp} and~\ref{thm:dp-dp2}) guarantees.

\subsection{DP-TS-UCB Algorithm}
Algorithm~\ref{alg:private} presents the pseudo-code of DP-TS-UCB. Let $c_0 = \sqrt{2\pi e}$. We input trade-off parameter $\alpha \in [0,1]$ and learning horizon $T$, and then we compute the sampling budget $\phi = c_0 T^{0.5(1-\alpha)}  \ln^{0.5(3-\alpha)}(T) $.
 Let $n_i(t-1)$ denote the number of observations that are used to compute the empirical estimate $\hat{\mu}_{i, n_i(t-1)}$ at the end of round $t-1$.

%\paragraph{Initialization.} 
\textbf{Initialize learning algorithm (Line~2).} There are several steps to initialize the learning algorithm. (1) We pull each arm $i \in [K]$ once to initialize %$n_{i} = 1$ and  compute 
each arm's empirical mean $\hat{\mu}_{i, n_i}$. Since the decisions in these rounds do not rely on any data, we do not have any privacy concerns. 
(2) As we  use  the {\emph{arm-specific epoch structure} (Figure~\ref{fig:arm_specific_epoch} describes the key ideas of this structure)} to process observations, we use $r_i$ to track arm $i$'s epoch progress and use $O_i$ to count the number of unprocessed observations in epoch $r_i$. We initialize $r_i = 1$ and $O_i = 0$. 
(3) Since we {\emph{can only draw at most $\phi$ mean reward models from each  Gaussian distribution}}, 
we use $h_i$ to count the  remaining Gaussian sampling budget at the end of round $t-1$, and $\text{MAX}_i$ to track the maximum value among these $\phi$ Gaussian mean reward models. Initially, we set $h_i = \phi$ and $\text{MAX}_i= 0$.

%\paragraph{}
\textbf{Decide learning models (Line~4 to Line~11).}  Let $\theta_i(t)$ denote arm $i$'s learning model in round $t \ge K+1$. Each $\theta_i(t)$ can either be \emph{a new Gaussian mean reward model} or \emph{some Gaussian mean reward model already used before}. 
To decide which case fits arm $i$ in round $t$, we check the value of $h_i$ to see whether drawing a new Gaussian mean reward from $\mathcal{N} \left(\hat{\mu}_{i, n_i(t-1)},  \ln^{\alpha}(T)/n_i(t-1)\right)$ is allowed: if $h_i \ge 1$,  we sample a new mean reward model $ \theta_{i, n_i}^{(h_i)} \sim \mathcal{N} \left(\hat{\mu}_{i, n_i(t-1)},  \ln^{\alpha}(T)/n_i(t-1)\right)$ and use it in the learning, i.e., $\theta_i(t) =  \theta_{i, n_i}^{(h_i)} $;
%$\theta_i(t) \sim \mathcal{N} \left(\hat{\mu}_{i, n_i(t-1)},  \ln^{\alpha}(T)/n_i(t-1)\right)$. 
if $h_i = 0$, we use $\theta_i(t) = \text{MAX}_i =\mathop{\max}_{h_i \in [\phi]} \theta^{(h_i)}_{i, n_i}$ in the learning as  we have all $\theta_{i,n_i}^{(1)}, \theta_{i,n_i}^{(2)}, \dotsc, \theta_{i,n_i}^{(\phi)}$ in hand already. 

 Our technical Lemma~\ref{lemma boost} below  shows that  the highest mean reward model 
 $\text{MAX}_i$ is analogous to  the upper confidence bound in UCB1~\citep{auer2002finite}. The usage of $\text{MAX}_i$ ensures sufficient exploration for the rounds when sampling new Gaussian mean reward models is not allowed.
 We can view DP-TS-UCB as a two-phase algorithm with a mandatory TS-Gaussian phase  and an optional UCB phase. Note that DP-TS-UCB itself does not explicitly construct upper confidence bounds; $\text{MAX}_i$ itself behaves like the upper confidence bound of arm $i$ in UCB1 in terms of achieving exploration. 
 %By connecting  exploration mechanisms in TS-Gaussian and UCB1 of \citet{auer2002finite},%{\color{red}The meaning of UCB... with probability, the used learning model is better than the true model...}
\begin{lemma}
\label{lemma boost}
Fix any observation number $s \ge 1$ and let $\theta_{i,s}^{(1)}, \dotsc, \theta_{i,s}^{(\phi)}$ be i.i.d. according to $\mathcal{N}\left(\hat{\mu}_{i,s}, \ln^{\alpha}(T)/s\right)$. We have $\mathbb{P} \left\{ \mathop{\max}_{h \in [\phi]} \theta^{(h)}_{i, s} \ge \mu_i \right\} \ge 1- O(1/T)$.
\end{lemma}

\textbf{Make a decision and collect data (Line~12).} With all learning models $\theta_i(t)$ in hand, the learning agent pulls the arm $i_t \in \arg \mathop {\max }_{i \in [K]} \theta_i(t)$ with the highest model value, observes $X_{i_t}(t)$ and increments the unprocessed observation counter $O_{i_t}$ by one. 
%\paragraph{Arm-specific epoch structure.} (steps highlighted by red) 

\textbf{Process collected data (Line 13 to Line 17).} To control the number of times any observation can be used when computing the empirical mean, %we use \text{arm-specific epoch structure} to ensure that any observation is used at most once when computing the empirical mean. 
we only update the empirical mean of the pulled arm $i_t$ when the number of unprocessed observations $O_{i_t} = 2^{r_{i_t}}$. After the update, we reset $h_{i_t}$, $ O_{i_t}$  and $\text{MAX}_{i_t}$,  and increment the epoch progress $r_{i_t}$ by one.

\textbf{Remark on Algorithm~\ref{alg:private}.} 
We use  data collected in epoch $r_i-1$ in a differentially private manner to guide the future data collection in epoch $r_i$. We have a mandatory TS-Gaussian phase where drawing Gaussian mean reward models is allowed and an optional UCB phase where the agent can only reuse the best Gaussian mean reward model in the mandatory TS-Gaussian phase. Separating all the rounds belonging to epoch $r_i$ into two possible phases controls the cumulative injected noise (and privacy loss) regardless of the epoch length.

 \subsection{Regret Analysis of DP-TS-UCB}
 In this section, we provide a regret analysis of Algorithm~\ref{alg:private}. 
 \begin{theorem}   
\label{thm: regret}
 The problem-dependent regret bound of DP-TS-UCB with trade-off parameter $\alpha \in [0,1]$ is 
 
 $\sum_{i: \Delta_i >0} O ( \frac{\ln \left( T^{0.5(3-\alpha)}   \Delta_i^2 \right) \ln^{\alpha}(T)}{ \Delta_i} +\frac{(3-\alpha) \ln  \ln (T) \cdot \ln^{\alpha}(T)}{ \Delta_i})$.
 
 The worst-case regret bound  of DP-TS-UCB with trade-off  parameter $\alpha \in [0,1]$ is $O (\sqrt{KT} \ln^{0.5(1+\alpha)}(T))$.% {\color{blue}add instance-dependent regret bounds}
\end{theorem}
Theorem~\ref{thm: regret} gives the following corollary immediately. 
\begin{corollary}
 DP-TS-UCB with trade-off  parameter $\alpha = 0$ achieves $\sum_{i: \Delta_i >0} O \left( \ln \left( T^{1.5} \Delta_i^2 \right) /\Delta_i\right) + O \left( \ln  \ln (T)/\Delta_i\right)$ problem-dependent regret bounds and $O (\sqrt{KT \ln(T)})$ worst-case regret bounds.
 \label{coro: regret}
\end{corollary}

\textbf{Discussion.}
%We first discuss about the cases where $\alpha =0 $ and $\alpha = 1$. %compare to TS-Gaussian... novel banded analysis
DP-TS-UCB with parameter $\alpha =0 $ can be viewed as a problem-dependent optimal bandit algorithm with theoretical guarantees lying between TS-Gaussian \citep{agrawalnear} and UCB1 \citep{auer2002finite}; %{\color{blue}The exploration mechanism in DP-TS-UCB blends the TS-Gaussian exploration mechanism and the UCB1 exploration mechanism.} 
the $\sum_{i: \Delta_i >0} O \left( \ln \left( T^{1.5} \Delta_i^2 \right) /\Delta_i\right) + O \left( \ln  \ln (T)/\Delta_i\right)$ bound  %(setting$\alpha =0 $) %satisfies Sub-UCB criteria introduced in \citet{lattimore2018refining}, and 
is better than the $\sum_{i: \Delta_i >0} O \left( \ln( T ) /\Delta_i\right)$ bound of UCB1, but it is slightly worse than the $\sum_{i: \Delta_i >0} O \left( \ln( T \Delta_i^2 ) /\Delta_i\right)$ bound of TS-Gaussian. %Although the regret bound of DP-TS-UCB with $\alpha =1 $ %gives a $\sum_{i: \Delta_i >0} O \left( \ln \left( T   \Delta_i^2 \right) \ln(T)/\Delta_i \right)  + O \left( \ln  \ln (T) \cdot  \ln(T)/\Delta_i\right)$ regret bound, which 
DP-TS-UCB with parameter $\alpha =0 $ 
is not optimal in terms of regret guarantees, but it offers a constant GDP guarantee (see Corollary~\ref{coro} in Section~\ref{sec: privacy analysis}).

%= \sum_{i \in [K]: \Delta_i >0} O \left( \ln(\phi ) /\Delta_i\right) + \sum_{i \in [K]: \Delta_i >0} O \left( \ln(T \Delta_i^2 ) /\Delta_i\right) $\phi = c_0 T^{0.5}  \ln^{1.5}(T)$

We sketch the proof of Theorem~\ref{thm: regret}. The full proof is deferred to Appendix~\ref{app: regret proof}. Since DP-TS-UCB lies in between TS-Gaussian and UCB1,  the regret analysis includes key ingredients extracted from both algorithms. 
\begin{proof}[Proof sketch of Theorem~\ref{thm: regret}]
 %the regret is from both the inaccurate estimation of both the sub-optimal arm $i$ and the optimal arm $1$.
Fix a sub-optimal arm $i$.  Let $L_i = O \left( \ln(\phi T \Delta_i^2 ) \ln^{\alpha}(T)/\Delta_i^2\right)$ indicate the number of observations needed to sufficiently observe  sub-optimal arm $i$. %Here, it means the Gaussian distribution is concentrated on the empirical mean and the empirical mean is around the true mean. 
We know that the total regret accumulated from arm $i$ before $i$ is sufficiently observed is at most $L_i \cdot \Delta_i$. 
%the number of observations needed to have accurate empirical estimate.
% We first decompose the regret based on whether the sub-optimal arm $i$ is sufficiently observed  or not. 
    By tuning $L_i$ properly, for all the rounds when  arm $i$ is observed sufficiently, the regret accumulated from arm $i$ can be upper bounded by
    \begin{equation}
        \begin{array}{l}
%\omega_1    & =  &  
\sum_{t=K+1}^{T} \mathbb{P}  \left\{i_t = i, \theta_i(t) \le \mu_i + 0.5 \Delta_i \right\} \quad,
      \end{array}
      \label{eq: reg}
    \end{equation}
where $\theta_i(t)$ can either be a fresh Gaussian mean reward model (TS-Gaussian phase, Line~6) or the highest Gaussian mean model used before (UCB phase, Line~9). 

   We further decompose (\ref{eq: reg}) based on whether  the optimal arm $1$ is in TS-Gaussian phase (Line~6) or UCB phase (Line~9). 
Define $\mathcal{T}_1(t)$ as the event that the optimal arm $1$   uses a fresh Gaussian mean reward model in round $t$, i.e., in TS-Gaussian phase, and let $\overline{\mathcal{T}_1(t)}$ denote the complement. %, i.e.,  the optimal arm $1$  uses the best Gaussian mean reward model in round $t$.
%$\text{MAX}_1 = \mathop{\max}_{h_1 \in [\phi]} \theta^{(h_1)}_{1, n_1(t-1)}$ in the learning, where $\theta^{(h_1)}_{1, n_1(t-1)}  \sim \mathcal{N}\left(\hat{\mu}_{1,n_1(t-1)}, \frac{\ln^{\alpha}(T)}{n_{1}(t-1)} \right)$ for each $h_1 \in [\phi]$.
     We have (\ref{eq: reg}) decomposed as 
\begin{equation}
    \begin{array}{ll}
&\sum_{t=K+1}^{T} \mathbb{P}  \left\{i_t = i, \theta_i(t) \le \mu_i + 0.5 \Delta_i , \mathcal{T}_1(t) \right\} \\
+& \\
 & \sum_{t=K+1}^{T} \mathbb{P}  \left\{i_t = i, \theta_i(t) \le \mu_i + 0.5 \Delta_i , \overline{\mathcal{T}_1(t)} \right\}\quad, 
\end{array}
\end{equation}
where, generally, the first term will use the regret analysis  of TS-Gaussian in \citet{agrawalnear} and the second term will use Lemma~\ref{lemma boost}. In Appendix~\ref{app: regret proof},  we present an improved analysis of TS-Gaussian and show that the regret of the first term is at most $O \left( \ln(\phi T \Delta_i^2 ) \ln^{\alpha}(T)/\Delta_i\right)$. 
%$O \left( \ln(\phi T \Delta_i^2 ) \ln^{\alpha}(T)/\Delta_i\right)$
The second term uses a union bound and Lemma~\ref{lemma boost}, and is at most $ \sum_{t=K+1}^{T} \mathbb{P}  \left\{i_t = i, \theta_i(t) \le \mu_i + 0.5 \Delta_i , \overline{\mathcal{T}_1(t)} \right\} 
         \le \sum_{t=K+1}^{T} 
\mathbb{P}  \left\{ \theta_1(t) \le \mu_1 , \overline{\mathcal{T}_1(t)} \right\} \le O(\ln(T))$.\end{proof}

\subsection{Privacy Analysis of DP-TS-UCB} \label{sec: privacy analysis}
\begin{table*}[h]
\centering
\caption{Summary of privacy and regret guarantees}
\label{tab:regret_privacy}
\resizebox{0.9\textwidth}{!}{
\begin{tabular}{|l|l|l|l|l|}
\hline
   & \textbf{Regret bounds} & \textbf{GDP guarantees} \\ \hline
TS-G
\footnotesize{\citep{agrawalnear}}           & $O \left(K\ln(T \Delta^2)/\Delta \right)$    & $O (T^{0.5})$                 \\ \hline
M-TS-G   {\footnotesize \citep{ou2024thompsonsamplingdifferentiallyprivate}}          & $bK +  O( c K \ln(T \Delta^2)/\Delta)$    & $O(\sqrt{T/(c(b+1))})$            \\ \hline

M-TS-G {\footnotesize(tune $b\text{,}c = O(T^{\gamma})$\text{,}$\gamma > 0$)}            & $  O( KT^{\gamma}  \ln(T \Delta^2)/\Delta)$    & $O (T^{0.5-\gamma})$
 \\ \hline
M-TS-G {\footnotesize(tune $b\text{,} c = O \left(\ln^{\alpha}(T)\right)$) }           & $ O( K\ln^{\alpha}(T)  \ln(T \Delta^2)/\Delta)$    & $O (T^{0.5}/\ln^{\alpha}(T))$
   \\ \hline
DP-TS-UCB {\footnotesize(Algorithm~\ref{alg:private})}& $ O( K\ln \left( T^{0.5(3-\alpha)}   \Delta^2 \right) \ln^{\alpha}(T)/ \Delta+ K \ln  \ln (T)  \ln^{\alpha}(T)/ \Delta)$    & $ O ( T^{0.25(1-\alpha)} \ln^{0.75(1-\alpha)}(T))$
\\ \hline
DP-TS-UCB {\footnotesize(tune $\alpha = 0$)}& $O ( K\ln \left( T^{1.5} \Delta^2 \right) /\Delta + K \ln  \ln (T)/\Delta)$    & $\tilde{O}\left(T^{0.25}\right)$  \\ \hline
DP-TS-UCB {\footnotesize(tune $\alpha = 1$)}& $O ( K\ln \left( T   \Delta^2 \right) \ln(T)/\Delta )+K \ln  \ln (T)   \ln(T)/\Delta)$    & $O(1)$                      \\ \hline
\end{tabular}
}
\end{table*}
This section provides the privacy analysis of Algorithm~\ref{alg:private}.
\begin{theorem}
\label{thm:dp-dp}
DP-TS-UCB with trade-off parameter $\alpha \in [0,1]$ 
satisfies  $ \sqrt{2c_0 T^{0.5(1-\alpha)} \ln^{1.5(1-\alpha)}(T)}$-GDP.
\end{theorem}
 Theorem~\ref{thm:dp-dp} gives the following corollary immediately.
\begin{corollary}
DP-TS-UCB with trade-off parameter $\alpha = 0$  satisfies $O \left(T^{0.25} \ln^{0.75}(T)\right)$-GDP; DP-TS-UCB with trade-off parameter $\alpha = 1$  satisfies $O \left(1\right)$-GDP. 
\label{coro}
\end{corollary}

\textbf{Discussion.} %{\color{red}add a table...} 
Together, Theorem~\ref{thm: regret} (regret guarantees) and Theorem~\ref{thm:dp-dp} (privacy guarantees) show that DP-TS-UCB is able to trade off privacy and regret. The privacy guarantee improves with the increase of trade-off parameter $\alpha$, at the cost of suffering more regret.

Table~\ref{tab:regret_privacy} summarizes privacy and regret  guarantees of TS-Gaussian \citep{agrawalnear}, M-TS-Gaussian \citep{ou2024thompsonsamplingdifferentiallyprivate}, and DP-TS-UCB.  From the results, even for the worst case, i.e., $\alpha = 0$,  DP-TS-UCB is still $\tilde{O} \left(T^{0.25}\right)$-GDP, which could be much better than the $O(\sqrt{T})$-GDP guarantee of  TS-Gaussian. %of \citet{agrawalnear}. %and  M-TS-Gaussian of \citet{ou2024thompsonsamplingdifferentiallyprivate}. %Our privacy improvement  comes from capping the number of random samples that can be drawn from a Gaussian distribution and using arm-specific epoch structure to process observations. 
%DP-TS-UCB with  $\alpha = 0$ uses the same variance as TS-Gaussian and achieves $\sum_{i: \Delta_i >0} O \left( \ln \left( T^{1.5} \Delta_i^2 \right) /\Delta_i\right) + O \left( \ln  \ln (T)/\Delta_i\right)$ regret bound.  From  Corollary~\ref{coro}, we know  it achieves $O \left(T^{0.25} \ln^{0.75}(T)\right)$-GDP. 
Since DP-TS-UCB with $\alpha = 1$ achieves 
%$\sum_{i \in [K]: \Delta_i >0} O \left( \ln( \ln T ) \ln (T) /\Delta_i\right) + O \left( \ln( T \Delta_i^2 ) \ln(T) /\Delta_i\right)$ regret bound and
a constant GDP guarantee,  increasing learning horizon $T$ does not increase privacy cost.
M-TS-Gaussian pre-pulls each arm $b$ times and uses $c/n_i$ as the Gaussian variance. Generally, it achieves $bK + \sum_{i: \Delta_i >0} O( c \log(T \Delta_i^2)/\Delta_i)$ regret bounds and satisfies  $\sqrt{T/(c(b+1))}$-GDP. 
%We first compare our DP-TS-UCB with 
%M-TS-Gaussian with $b, c$ as constants achieves  $ \sum_{i: \Delta_i >0} O( \log(T \Delta_i^2)/\Delta_i)$ regret bound and satisfies $O\left(\sqrt{T}\right)$-GDP.  % and $c = \ln^{\alpha}(T)$. 
 By tuning $b,c = O\left(\ln^{\alpha}(T)\right)$, M-TS-Gaussian achieves $ \sum_{i: \Delta_i >0} O(\ln^{\alpha}(T) \log(T \Delta_i^2)/\Delta_i)$ regret bounds (almost the same as DP-TS-UCB's regret bounds), but satisfying $O (\sqrt{T}/\ln^{\alpha}(T) )$-GDP guarantees, which could be much worse than the $\tilde{O} \left(T^{0.25} \right)$-GDP guarantees of DP-TS-UCB. % when $T$ is sufficiently large.
  By tuning $b, c=O \left(T^{\gamma}\right)$, where $\gamma > 0$, M-TS-Gaussian achieves  $ \sum_{i: \Delta_i >0} O(T^{\gamma}\log(T \Delta_i^2)/\Delta_i)$ regret bounds and satisfies $O (\sqrt{T^{1-2\gamma}})$-GDP. Although the GDP guarantee is improved to be in the order of $\sqrt{T^{1-2\gamma}}$, the regret bound may be  worse than DP-TS-UCB's bounds due to the existence of the $T^{\gamma}$ term. For example, when setting $\gamma = 0.25$, M-TS-Gaussian is $O(T^{0.25})$-GDP, but it has a $ \sum_{i: \Delta_i >0} O(T^{0.25}\log(T \Delta_i^2)/\Delta_i)$ regret bound, which will not be problem-dependent optimal.

Since the classical $(\varepsilon, \delta)$-DP notion is more interpretable, we translate GDP guarantee presented in Theorem~\ref{thm:dp-dp} into $(\varepsilon, \delta)$-DP guarantees by using Theorem~\ref{the: duality}. 
\begin{theorem}
\label{thm:dp-dp2}
   % duality between the standard DP and GDP.
DP-TS-UCB  is $(\varepsilon, \delta(\varepsilon))$-DP for all $\varepsilon \ge 0$, where $\delta(\varepsilon) = \Phi \left(-\frac{\varepsilon}{\sqrt{2 \phi}} + \frac{\sqrt{2 \phi}}{2} \right) - e^{\varepsilon} \cdot \Phi \left(-\frac{\varepsilon}{\sqrt{2 \phi}} - \frac{\sqrt{2 \phi}}{2} \right)$, where $\phi = c_0 T^{0.5(1-\alpha)}  \ln^{0.5(3-\alpha)}(T)$.
    \end{theorem}

  \begin{proof}%[Proof of Theorem~\ref{thm:dp-dp2}]
   Directly using Theorem~\ref{the: duality} concludes the proof.
%     Theorem~\ref{thm:dp-dp} gives $\eta =  \sqrt{2c_0 T^{0.5(1-\alpha)} \ln^{1.5(1-\alpha)}(T)}$. Then, using Theorem~\ref{the: duality} concludes the proof.  
    \end{proof}

The proof for Theorem~\ref{thm:dp-dp} relies on the following composition theorem and post-processing theorem of GDP.
\begin{theorem}[GDP composition  \citep{dong2022gaussian}] The $m$-fold composition of $\eta_j$-GDP mechanisms is $\sqrt{\eta_1^2 + \dotsc + \eta^2_m}$-GDP.
    \label{Thm: GDP composition}
    \end{theorem}
\begin{theorem}[GDP Post-processing \citep{dong2022gaussian}]
 If a mechanism $\mathcal{A}$ is $\eta$-GDP, its post-processing is also $\eta$-GDP. 
    \label{thm: post-proc}
\end{theorem}

%We first present privacy guarantees for our DP-TS-UCB.
%Let $c_0 := 2\sqrt{2\pi e}$.

\begin{proof}[Proof of Theorem~\ref{thm:dp-dp}] Fix any two neighbouring reward sequences  $X_{1:T} = \left(X_1, \dotsc, X_{\tau} \dotsc, X_T \right)$ and $X'_{1:T} = \left(X_1,  \dotsc, X'_{\tau}, \dotsc X_T \right)$, where
 the complete reward vector in round $\tau$ is changed. Under the bandit feedback model, this change only  impacts the empirical mean of the arm pulled in round $\tau$, that is arm $i_{\tau}$. Name $i_{\tau} = j$: based on the arm-specific epoch structure (Figure~\ref{fig:arm_specific_epoch}),
  the observation $X_{j}(\tau)$ will only be used once for computing the empirical mean of arm $j$ at the end of some future round,  which is the last round of some   epoch $r_j -1$ associated with arm $j$.
%some future round due to the usage of arm-specific structure. 

We have one Gaussian distribution constructed using $X_{j}(\tau)$ at the beginning of epoch $r_j$. % and $X_{j}(\tau)$ will be abandoned in epoch $r_j +2$.
If arm $j$ only has the mandatory TS-Gaussian phase in epoch $r_j$, we draw at most
 $\phi$ Gaussian mean reward models from that constructed Gaussian distribution. From Lemma~5 of \citet{ou2024thompsonsamplingdifferentiallyprivate}, we know DP-TS-UCB is $\sqrt{1/\ln^{\alpha}(T)}$-GDP in each round in the mandatory TS-Gaussian phase.  From Theorem~\ref{Thm: GDP composition}, we know the GDP composition over at most $\phi$ rounds is $\sqrt{\phi  /\ln^{\alpha}(T)}$-GDP. Note that $X_{j}(\tau)$ will not be used to construct Gaussian distributions starting from epoch $r_j +1$ to the end of learning due to the usage of arm-specific epoch structure, i.e., we abandon $X_{j}(\tau)$ at the end of epoch $r_j$.

If arm $j$  has both the mandatory TS-Gaussian phase and the optional UCB phase in epoch $r_j$, for the mandatory TS-Gaussian phase, DP-TS-UCB is $\sqrt{\phi  /\ln^{\alpha}(T)}$-GDP; for the optional UCB phase, DP-TS-UCB is also $\sqrt{\phi  /\ln^{\alpha}(T)}$-GDP, as by post-processing Theorem~\ref{thm: post-proc}, the maximum $\text{MAX}_j $  of $\phi$ Gaussian mean reward models  is $\sqrt{\phi /\ln^{\alpha}(T) }$-GDP.  
Composing the privacy guarantees in these two phases concludes the proof. 
\end{proof}

\section{Experimental Results}

\begin{figure}[h]
    \centering
    \subfigure[The impact of learning horizon $T$ and trade-off parameter $\alpha$ on the regret by the end of round $T$.]{
    \includegraphics[width=0.9\linewidth]{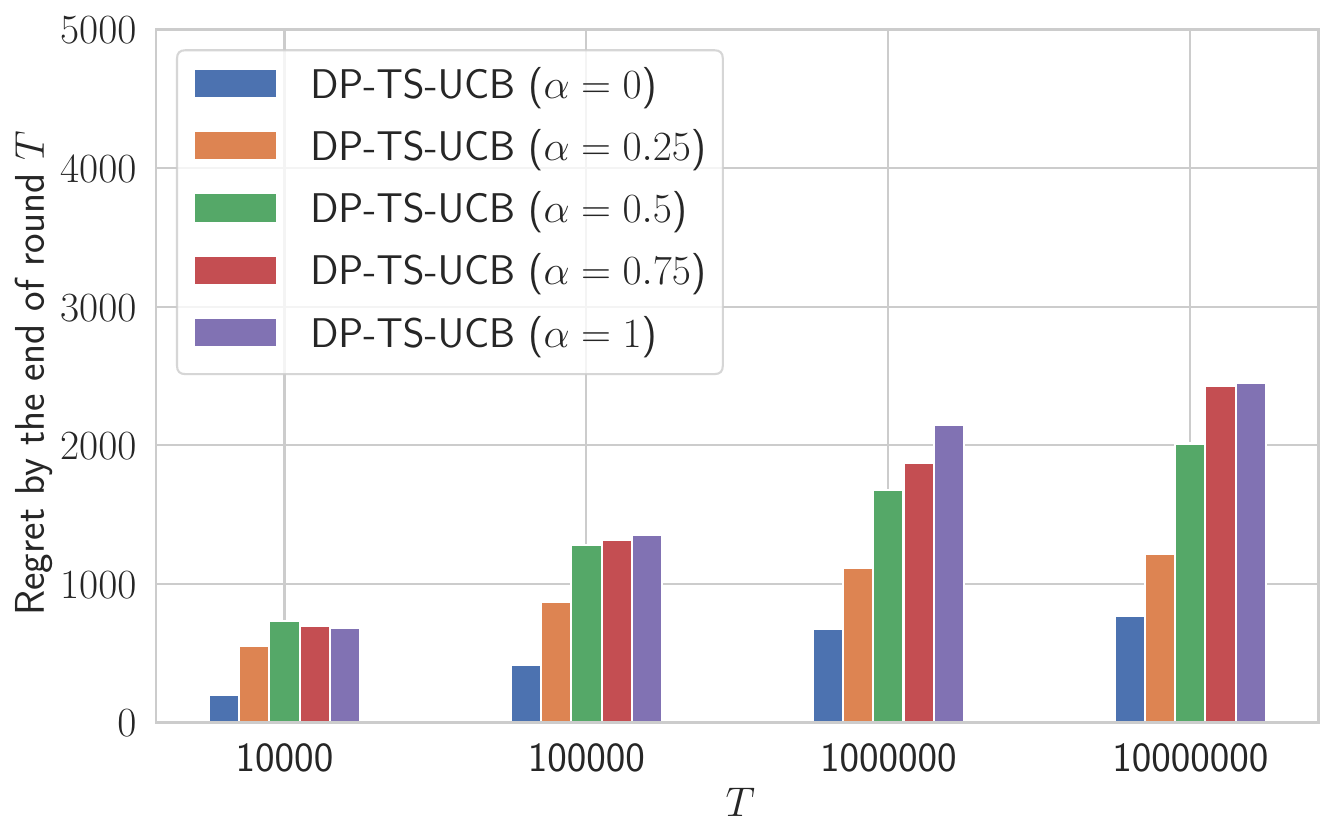}
    \label{fig:varyingAlphaT}
    }
    \subfigure[The impact of learning horizon $T$ and trade-off parameter $\alpha$ on the GDP parameter $\eta$ by the end of round $T$.]{
        \includegraphics[width=0.9\linewidth]{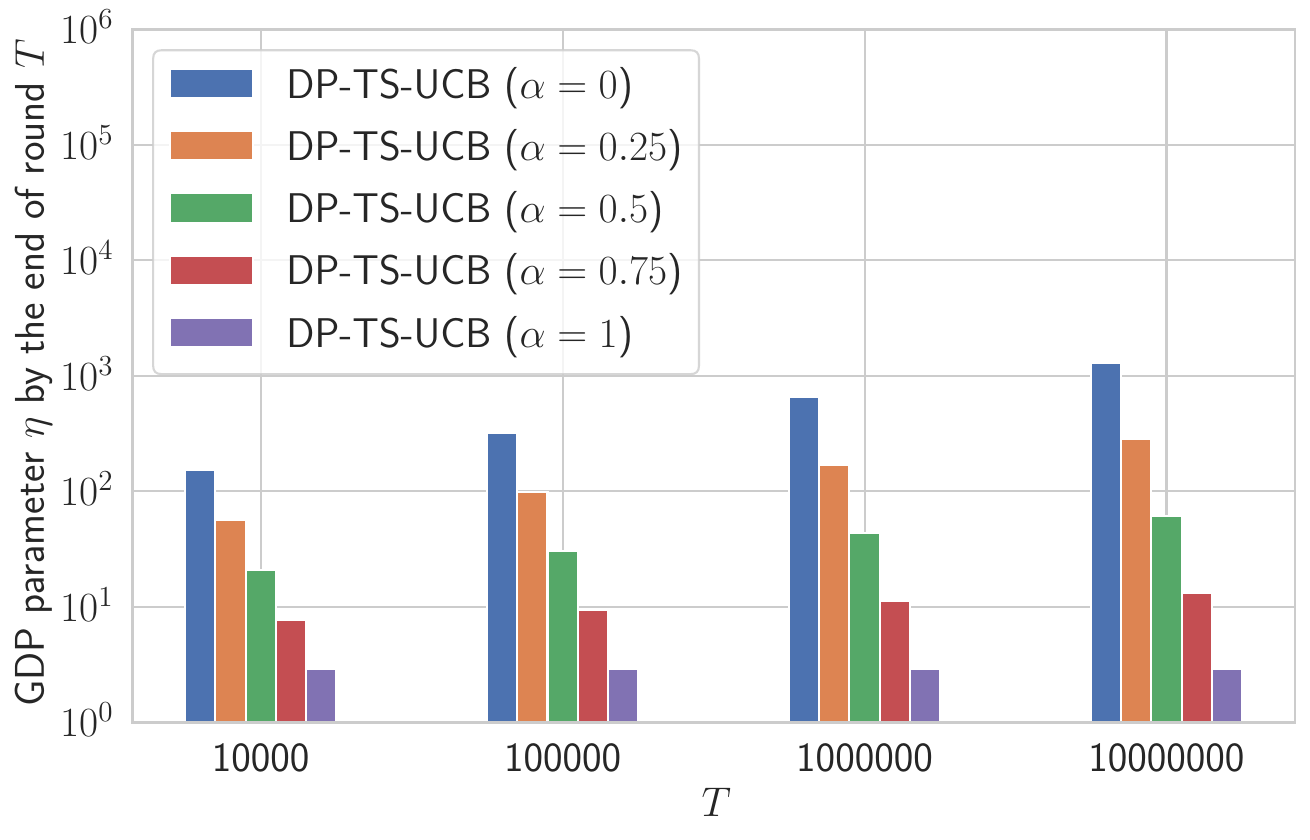}
        \label{fig:eta_varyingAlpha}
    }
    \caption{DP-TS-UCB's privacy  vs regret with different   $\alpha$ and $T$.} \label{fig:dp_ts_ucb_privacy_vs_regret}
\end{figure}

The setup consists of five arms with Bernoulli rewards. We set the mean rewards as $[0.95, 0.75, 0.55, 0.35, 0.15]$. We first analyze DP-TS-UCB's privacy and regret across different values of $\alpha$ and $T$. Then, we compare DP-TS-UCB with M-TS-Gaussian~\citep{ou2024thompsonsamplingdifferentiallyprivate} from two perspectives: (1) \textbf{Privacy cost under equal regret}; (2) \textbf{Regret under equal privacy guarantee}.  
We also compare with $(\varepsilon, 0)$-DP algorithms, including DP-SE~\citep{sajed2019optimal}, Anytime-Lazy-UCB~\citep{hu2021near}, and Lazy-DP-TS~\citep{hu2022near} for $\varepsilon = 0.5$, which can be found in Appendix~\ref{app:epsi}. % which use the Laplace mechanism for noise injection.  
All the experimental results are an average of $20$ independent runs on a MacBook Pro with M1 Max and 32GB RAM.

\subsection{Privacy and Empirical Regret  of DP-TS-UCB with Different Values of $\alpha$ and $T$}

The performance of DP-TS-UCB in terms of the privacy guarantees and regret across different values of  $\alpha$  and time horizons  $T$ are shown in Figure~\ref{fig:dp_ts_ucb_privacy_vs_regret}. The results reveal a tradeoff between regret minimization and privacy preservation: increasing  $\alpha$  leads to a stronger privacy guarantee, reflected in a lower GDP parameter  $\eta$, but at the cost of higher regret. However, when  $\alpha = 1$, the privacy guarantee becomes constant, meaning that increasing  $T$  no longer deteriorates the privacy protection of DP-TS-UCB.
%we obtain a constant privacy guarantee at a cost of only a marginal increase in regret (the regret gap between $\alpha = 0.75$ and $\alpha=1$ is relatively small).

\subsection{Privacy and Empirical Regret Comparison under the Same Theoretical Regret Bound}
%From Table~\ref{tab:regret_privacy}, 
Since DP-TS-UCB with parameter $\alpha$ and M-TS-Gaussian with parameters $b=0$, $c = 5\ln^{\alpha}(T)$) share the same theoretical regret bound, we now present empirical regret and privacy guarantees for different values of $\alpha = \{0, 0.25, 0.5, 0.75, 1 \}$. We set $T = 10^6$. Figure~\ref{fig:fixregret_varyingalpha} shows that DP-TS-UCB incurs lower empirical regret than M-TS-Gaussian, whereas %under the same theoretical regret guarantee. Additionally, 
Figure~\ref{fig:varyingalpha_eta} shows that DP-TS-UCB achieves better privacy. % with a lower GDP parameter $\eta$.  

\begin{figure}[h]
    \centering
    \subfigure[Regret by the end of round $T$ of DP-TS-UCB and M-TS-Gaussian with different parameters.]{
        \includegraphics[width=0.9\linewidth]{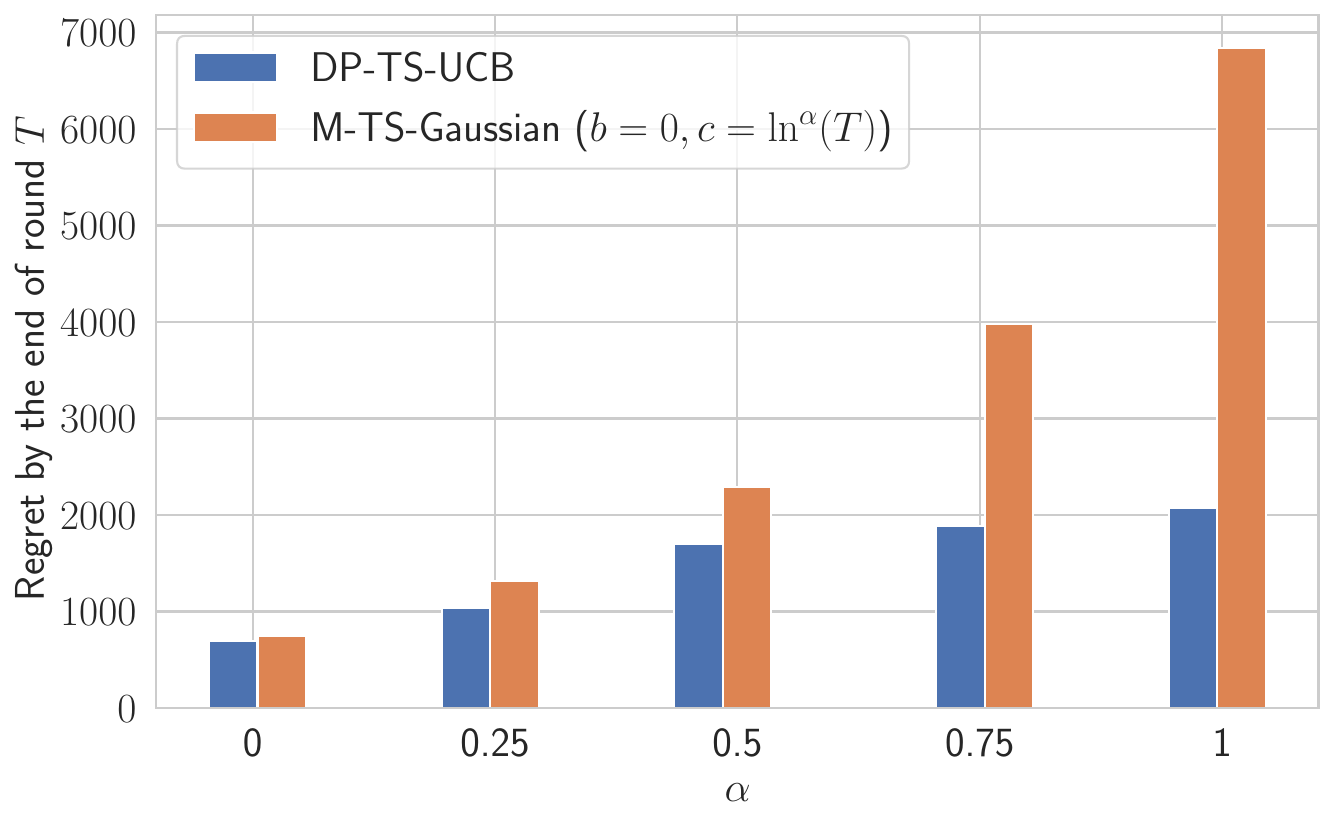}
        \label{fig:fixregret_varyingalpha}
    }
    \hfill
    \subfigure[GDP parameter $\eta$ by the end of round $T$ of DP-TS-UCB and M-TS-Gaussian  with different parameters.]{
        \includegraphics[width=0.9\linewidth]{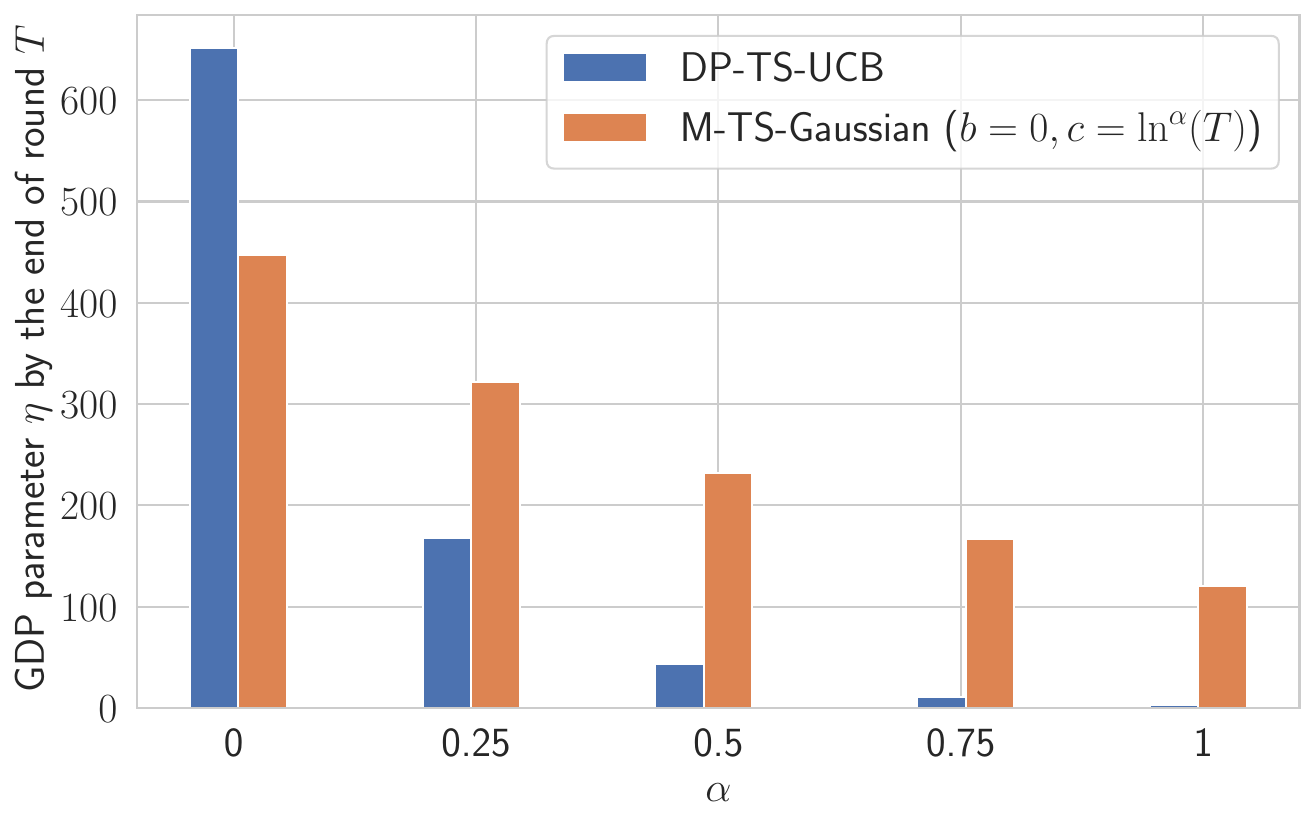}
        \label{fig:varyingalpha_eta}
    }
    \caption{The performance of DP-TS-UCB and M-TS-Gaussian under the same theoretical regret bound.}
    \label{fig:main-fig}
\end{figure}

\subsection{Empirical Regret Comparison  under the Same Privacy Guarantee}\label{sc:privacy}
M-TS-Gaussian satisfies a $\sqrt{T/(c(b+1))}$-GDP guarantee, while DP-TS-UCB satisfies $\sqrt{2c_0 T^{0.5(1-\alpha)} \ln^{1.5(1-\alpha)}(T)}$-GDP. Thus, we let $c = \sqrt{\frac{1}{2c_0 (b+1)}T^{0.5(1+\alpha)} \ln^{-1.5(1-\alpha)}T}$ for any $b$ of M-TS-Gaussian to ensure the same privacy guarantees as DP-TS-UCB. 
We compare their empirical regret over $T = 10^6$ rounds under two privacy settings determined by $\alpha$ for both algorithms. For each $\alpha$, we select $b$ from $\{0, 1, 500, 1000, 2000, 5000, 100000\}$ to minimize regret of  M-TS-Gaussian (see Appendix~\ref{appx:mtsg_pars}).  

\textbf{$\sqrt{2c_0 T^{0.5} \ln^{1.5} T}$-GDP Guarantee ($\alpha = 0$).}  
The optimal M-TS-Gaussian parameters are $b = 1$ and $c = 1.18$. As shown in Figure~\ref{fig:tregret_Bernoulli_without_epsi}, M-TS-Gaussian slightly outperforms DP-TS-UCB, but the empirical regret gap is small.  

\textbf{$\sqrt{2c_0}$-GDP Guarantee ($\alpha = 1$).}  
For this setting, the best M-TS-Gaussian parameters are $b = 2000$ and $c = 60.46$. However, DP-TS-UCB achieves lower regret, significantly outperforming M-TS-Gaussian, as shown in Figure~\ref{fig:c0regret_Bernoulli_without_epsi}.

\begin{figure}[h]
    \centering
    \subfigure[Empirical regret for  $O(\sqrt{T^{0.5} \ln^{1.5} T})$ -GDP  ($\alpha = 0$).]{
        \includegraphics[width=0.9\linewidth]{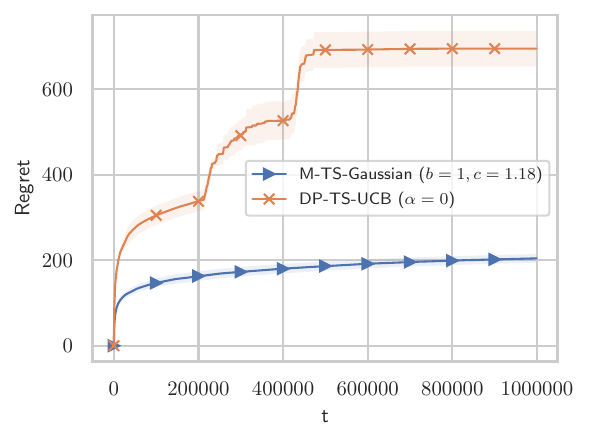}
        \label{fig:tregret_Bernoulli_without_epsi}
    }
    \subfigure[Empirical regret for $\sqrt{2c_0}$-GDP ($\alpha = 1$).]{
        \includegraphics[width=0.9\linewidth]{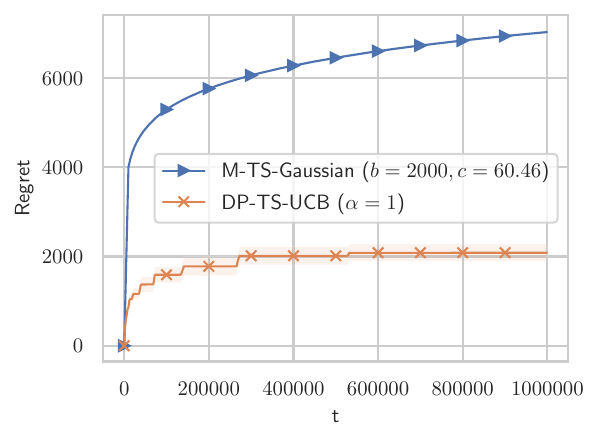}
        \label{fig:c0regret_Bernoulli_without_epsi}
    }
    \caption{The regret of DP-TS-UCB and M-TS-Gaussian under the same privacy guarantee with $\alpha = 0$ and $\alpha = 1$.}
    \label{fig:regret_same_privacy}
\end{figure}

\section{Conclusion}
This paper presents a novel private stochastic bandit algorithm DP-TS-UCB (Algorithm~\ref{alg:private}) by leveraging the connection between exploration mechanisms in TS-Gaussian and UCB1. 
We first %connect  TS-Gaussian  with Gaussian mechanism  and Gaussian differential privacy (GDP)  and 
show that DP-TS-UCB satisfies $\tilde{O}( T^{0.25(1-\alpha)})$-GDP and then we translate this GDP guarantee to the classical $(\varepsilon, \delta)$-DP guarantees by using duality between these two privacy notions. % GDP and $(\varepsilon, \delta)$-DP notions. %, the usage of arm-specific ...
  Corollary~\ref{coro: regret} and Corollary~\ref{coro} show that  DP-TS-UCB with parameter $\alpha =0$  achieves the optimal $O (K \ln(T)/\Delta)$ problem-dependent regret bounds and the near-optimal $O (\sqrt{KT \ln T} )$ worst-case regret bounds, and satisfies $\tilde{O}\left(T^{0.25} \right)$-GDP. This privacy guarantee could be much better than the $O(\sqrt{T})$-GDP guarantees achieved by TS-Gaussian and M-TS-Gaussian of \citet{ou2024thompsonsamplingdifferentiallyprivate}.  We conjecture that our privacy improvement is at the cost of the anytime property of the learning algorithm and the worst-case regret bounds. Note that both TS-Gaussian and M-TS-Gaussian are anytime and % but not DP-TS-UCB, and the worst-case regret bound. Note that both TS-Gaussian and M-TS-Gaussian 
  achieve $O (\sqrt{KT \ln K} )$ worst-case regret bounds, whereas our DP-TS-UCB is not anytime and  achieves only $O (\sqrt{KT \ln T} )$ worst-case regret bounds. 
If we know the maximum mean reward gap $\Delta_{\max} = \mathop{\max}_{i \in [K]}\Delta_i$ in advance, by slightly modifying the theoretical analysis, we know a better choice of $\phi$ should be  the one depending on $\Delta_{\max}$. Tuning $\phi$ that depends on $\Delta_{\max}$ will provide problem-dependent GDP guarantees. % should also dependent on the problem-dependent parameter $\Delta_{\max}$. 
 This intuition motivates us to develop  private algorithms that achieve problem-dependent GDP guarantees as the main future work. %Establishing lower bounds for our proposed DP-TS-UCB algorithm is also an interesting avenue for our future work. 

\section*{Acknowledgements}
Bingshan Hu is grateful for the funding support from the Natural Sciences and Engineering Resource Council of Canada (NSERC),
the Canada CIFAR AI chairs program,
and the UBC Data Science Institute. 
Zhiming Huang would like to acknowledge funding support from NSERC and the British Columbia Graduate Scholarship. Tianyue H. Zhang is grateful for the support from Canada CIFAR AI chairs program and Samsung Electronics Co., Limited. Mathias Lécuyer is grateful for the support of NSERC with reference number RGPIN-2022-04469.
Nidhi Hegde would like to acknowledge funding support from the Canada CIFAR AI Chairs program.
 \section*{Impact Statement}
Privacy-preserving sequential decision-making is important in modern interactive machine learning systems, particularly in bandit learning and its general variant reinforcement learning (RL). Our work contributes to this field by proposing a novel differentially private bandit algorithm that connects classical algorithms in the RL community and DP community.  
Understanding the interplay between  decision-making algorithms like Thompson Sampling, and privacy mechanisms and notions is fundamental to advance the deployment of RL algorithms using sensitive data. 

\bibliography{example_paper}
\bibliographystyle{icml2025}

%%%%%%%%%%%%%%%%%%%%%%%%%%%%%%%%%%%%%%%%%%%%%%%%%%%%%%%%%%%%%%%%%%%%%%%%%%%%%%%
%%%%%%%%%%%%%%%%%%%%%%%%%%%%%%%%%%%%%%%%%%%%%%%%%%%%%%%%%%%%%%%%%%%%%%%%%%%%%%%
% APPENDIX
%%%%%%%%%%%%%%%%%%%%%%%%%%%%%%%%%%%%%%%%%%%%%%%%%%%%%%%%%%%%%%%%%%%%%%%%%%%%%%%
%%%%%%%%%%%%%%%%%%%%%%%%%%%%%%%%%%%%%%%%%%%%%%%%%%%%%%%%%%%%%%%%%%%%%%%%%%%%%%%
\newpage
\appendix
\onecolumn
The appendix is organized as follows.
\begin{enumerate}
    \item Useful facts are provided in Appendix~\ref{app: facts};
    \item Proofs  for Lemma~\ref{lemma boost} is presented in Appendix~\ref{app: lemma proof};
    \item Proofs for Lemma~\ref{UBC 2} is presented in Appendix~\ref{lemma old proof};
    \item Proofs for Theorem~\ref{thm: regret} is presented in Appendix~\ref{app: regret proof};
    \item Additional experimental results are presented in Appendix~\ref{app: more exp}.
    
\end{enumerate}

\section{Useful facts} \label{app: facts}
%{\color{blue}Zhiming has verified the proof of this lemma.}
\begin{fact}
    For any $T > e^3$, for any $\alpha \in [0,1]$, we have
  $\ln^{1-\alpha} (T) \le (1-\alpha)\ln(T) +1$.
    \label{Fact: calculus}
\end{fact}
\begin{proof}
    Let function $f(\alpha) = (1-\alpha)\ln(T) +1-\ln^{1-\alpha} (T)$, where variable $ \alpha \in [0,1]$. 
        Then, we have $f'(\alpha) = -\ln(T)+\ln^{1-\alpha}(T)\ln(\ln(T))$.
    It is not hard to verify that $f'\left( \frac{\ln(\ln(\ln(T)))}{\ln(\ln(T))} \right)=0$. 
        The fact that  $f'(\alpha) \ge 0$ when $\alpha \in \left[0, \frac{\ln(\ln(\ln(T)))}{\ln(\ln(T))} \right]$ gives  $f(\alpha) \ge f(0) = 1 >0$ for any  $\alpha \in \left[0, \frac{\ln(\ln(\ln(T)))}{\ln(\ln(T))} \right]$. Similarly, the fact that 
   $f'(\alpha) \le 0$ when $\alpha \in \left[ \frac{\ln(\ln(\ln(T)))}{\ln(\ln(T))},1 \right]$ gives
   $f(\alpha) \ge f(1) = 0$ for any  $\alpha \in \left[ \frac{\ln(\ln(\ln(T)))}{\ln(\ln(T))},1 \right]$. Therefore, we have $f(\alpha) \ge 0$ for any $\alpha \in [0,1]$. %, which concludes the proof.
\end{proof}
\begin{fact}[Hoeffding's inequality]
  Let $X_1, X_2, \dotsc, X_n$ be $n$ independent random variables with support $[0,1]$. Let ${\mu}_{1:n} = \frac{1}{n}\sum_{i=1}^{n}X_i$. Then, for any $a >0$, we have
   $ \mathbb{P} \left\{\left|{\mu}_{1:n} -\mathbb{E}\left[{\mu}_{1:n}\right]\right|  \ge a  \right\} \le 2e^{-2na^2}$.
    \label{Hoeffding}
\end{fact}
\begin{fact}[Concentration and anti-concentration bounds of Gaussian distributions] For a Gaussian distributed random variable $Z$ with mean $\mu$ and variance $\sigma^2$, for any $z > 0$, we have
\begin{equation}
\begin{array}{l}
 \mathbb{P} \left\{Z > \mu + z \sigma \right\} \le \frac{1}{2}e^{- \frac{z^2}{2}}, \quad \mathbb{P} \left\{Z < \mu - z \sigma \right\} \le \frac{1}{2}e^{- \frac{z^2}{2}}\quad,
 \end{array}
    \label{Fact 2}
\end{equation}
and 
\begin{equation}
\begin{array}{l}
  \mathbb{P} \left\{Z > \mu + z \sigma \right\} \ge \frac{1}{\sqrt{2\pi}} \frac{z}{z^2+1} e^{- \frac{z^2}{2}} \quad.
  \end{array}
     \label{Fact 1}
\end{equation}
\end{fact}

% \section{Proofs for Theorem~\ref{thm:dp-dp}}
% \begin{theorem}
%     (Theorem~2.7 of \citet{dong2022gaussian}). Define the Gaussian mechanism $\mathcal{M}$ that operates on a statistic $\theta$ as $\mathcal{M}(S) = \theta(S) + Z$, where $Z \sim \mathcal{N}(0, \text{sens}(\theta)^2/\eta^2)$. Then, $\mathcal{M}$ is $\eta$-GDP.
% \end{theorem}\

\section{Proofs for Lemma~\ref{lemma boost}} \label{app: lemma proof}
% \paragraph{Restatement of Lemma~\ref{lemma boost}.}
% Fix any integer $s \ge 1$ and let $\theta_{i,s}^{(1)}, \dotsc, \theta_{i,s}^{(\phi)}$ be i.i.d. according to $\mathcal{N}\left(\hat{\mu}_{i,s}, \ln^{\alpha}(T)/s\right)$. We have $\mathbb{P} \left\{ \mathop{\max}_{h \in [\phi]} \theta^{(h)}_{i, s} \ge \mu_i \right\} \ge 1- O(1/T)$.

\begin{proof}[Proof of Lemma~\ref{lemma boost}]

Let  $\mathcal{E}_{i,s}^{\mu}$ denote the event that 
%the confidence interval holds, i.e.,  
$\left| \hat{\mu}_{i,s} - \mu_i \right| \le \sqrt{\ln(T )/s} $ holds. Let $ \overline{\mathcal{E}_{i,s}^{\mu} }$ denote the complement.  We have
      \begin{equation}
        \begin{array}{lll}
             \mathbb{P} \left\{ \mathop{\max}_{h \in [\phi]} \theta^{(h)}_{i, s} \le \mu_i \right\}  
         &  \le &\mathbb{P} \left\{\mathcal{E}_{i,s}^{\mu} \right\} \mathbb{P} \left\{ \mathop{\max}_{h \in [\phi]} \theta^{(h)}_{i, s} \le \mu_i \mid \mathcal{E}_{i,s}^{\mu} \right\} + \mathbb{P} \left\{ \overline{\mathcal{E}_{i,s}^{\mu} }\right\}  \\
            &  \le &\mathop{\prod}\limits_{h \in [\phi]} \mathbb{P} \left\{ \theta^{(h)}_{i, s} \le \hat{\mu}_{i,s} + \sqrt{\ln(T)/s} \mid \mathcal{E}_{i,s}^{\mu} \right\} + 2 e^{-2 \ln(T)}  \\
%           \le & \mathbb{P} \left\{ \mathop{\max}_{h \in [\phi]} \theta^{(h)}_{i, s} \le \hat{\mu}_{i,s} + \sqrt{\ln(T )/s} \mid \mathcal{E}_{i,s}^{\mu} \right\}  + 2 e^{-2 \ln(T \Delta_i^2)} \\
%          \le &\mathop{\prod}\limits_{h \in [\phi]}  \mathbb{P} \left\{ \theta^{(h)}_{i, s} \le \hat{\mu}_{i,s} + \sqrt{\ln(T )/s} \mid \mathcal{E}_{i,s}^{\mu} \right\}  + 2/T^2 \\
          &  = &\mathop{\prod}\limits_{h \in [\phi]} \left( 1-\mathbb{P}\left\{ \theta_{1,s}^{(h)} > \hat{\mu}_{1,s} + \sqrt{\ln^{1-\alpha}(T )\ln^{\alpha}(T )/s} \mid \mathcal{E}_{i,s}^{\mu} \right\} \right)+2/T^2 \\
           &      \le^{(a)} & \mathop{\prod}\limits_{h \in [\phi]} \left(1 - \frac{1}{\sqrt{2\pi}} \cdot \frac{\sqrt{\ln^{1-\alpha}(T)}}{\ln^{1-\alpha}(T)+1} e^{-0.5 { \ln^{1-\alpha}(T)}}  \right) + 2/T^2 \\
                 &       \le^{(b)} &  \left(1 - \frac{1}{\sqrt{2\pi}} \cdot \frac{\sqrt{\ln^{1-\alpha}(T)}}{\ln^{1-\alpha}(T)+1} e^{-0.5 {\left((1-\alpha)\ln(T)+1 \right)}}  \right)^{\phi} +2/T^2\\
                   &      = &  \left(1 - \frac{1}{\sqrt{2\pi e}} \cdot \frac{\sqrt{\ln^{1-\alpha}(T)}}{\ln^{1-\alpha}(T)} e^{-0.5 \left((1-\alpha)\ln(T) \right)}  \right)^{\phi} +2/T^2\\
                     &       \le^{(c)} & e^{- \phi / \sqrt{2\pi e} \cdot \frac{1}{\sqrt{\ln^{1-\alpha}(T)}} \cdot \frac{1}{T^{0.5(1-\alpha)}}} +2/T^2 \\
                       &     = & e^{- \sqrt{2\pi e}T^{0.5(1-\alpha)}  \ln^{0.5(3-\alpha)}(T)  / \sqrt{2\pi e} \cdot \frac{1}{\sqrt{\ln^{1-\alpha}(T)}} \cdot \frac{1}{T^{0.5(1-\alpha)}}} +2/T^2 \\
           %   \le & 1/T\\
         &  \le & 3/T,
           
        \end{array}
    \end{equation}
   where step $(a)$ uses the anti-concentration bound shown in (\ref{Fact 1}), step (b) uses $\ln^{1-\alpha}(T) \le (1-\alpha)\ln(T)+1 $ shown in Fact~\ref{Fact: calculus}, and step (c) uses $(1-x) \le e^{-x}$.
    \end{proof}

\section{Proofs for Lemma~\ref{UBC 2}} \label{lemma old proof}

For the case where $\alpha = 0$, Lemma~\ref{UBC 2} below is an improved version of Lemma~2.13 in \citet{agrawalnear} and our new results imply both improved problem-dependent and problem-independent regret bounds for Algorithm~2 in \citet{agrawalnear}. Assume $T\Delta_i^2 > e$. 
\begin{lemma}\label{UBC 2}
%Let $\tau_s^{(1)}$ be the round when the $s$-th pull of the optimal arm $1$ occurs and 

Let $\theta_{1,s} \sim \mathcal{N}\left(\hat{\mu}_{1,s}, \ \frac{\ln^{\alpha}(T)}{s}\right)$. Then, for any integer $s \ge 1$, we have
\begin{equation}
\begin{array}{lll}
\mathbb{E}_{\hat{\mu}_{1,s}}\left[\frac{1}{\mathbb{P} \left\{ \theta_{1,s} > \mu_1 \mid \hat{\mu}_{1,s} \right\} }-1   \right] &\le& 12.34 \quad.
\end{array}
\label{WWW1}
\end{equation}
Also, for any integer $s \ge \frac{4(1+\sqrt{2})^2 \ln(T \Delta_i^2) \ln^{\alpha}(T)}{\Delta_i^2}$, we have
\begin{equation}
\begin{array}{lll}
\mathbb{E}_{\hat{\mu}_{1,s}}\left[\frac{1}{\mathbb{P} \left\{ \theta_{1,s} > \mu_1 - \frac{\Delta_i}{2}\mid \hat{\mu}_{1,s} \right\} }-1   \right] &\le& \frac{72}{T\Delta_i^2}\quad.
\end{array}
\label{WWW11}
\end{equation}
\end{lemma}

\begin{proof}%{of Lemma~\ref{UBC 2}:} 
%Here, $s$ is the number of observations of the optimal arm $1$. For any $s \ge 1$, we upper bound $ \mathbb{E}\left[\frac{1}{\mathbb{P} \left\{ \theta_{1,s} > \mu_1 \mid \hat{\mu}_{1,s}  \right\} } \right]-1  $, where the randomness is taken over the empirical mean $\hat{\mu}_{1,s}$ of the optimal arm $1$.

For the result shown in (\ref{WWW1}),
we analyze two cases: $s = 1$ and $s \ge 2$.
For $s = 1$, we have 
\begin{equation}
    \begin{array}{l}
        \text{LHS of (\ref{WWW1})} =     \mathbb{E}\left[\frac{1}{\mathbb{P} \left\{ \theta_{1,s} > \mu_1 \mid \hat{\mu}_{1,s}  \right\} } \right]-1    

%        &  = & \mathbb{E}\left[\frac{1}{\mathbb{P} \left\{ \theta_{1,s} > \mu_1 - \frac{\Delta_i}{2} + \sqrt{\frac{1}{1}} - \sqrt{\frac{1}{1}} \mid \hat{\mu}_{1,s}  \right\} } \right]-1 \\
                    \le^{(a)}   \mathbb{E}\left[\frac{1}{\mathbb{P} \left\{ \theta_{1,s} > \hat{\mu}_{1,s} + \sqrt{\frac{\ln^{\alpha}(T)}{1}}  \mid \hat{\mu}_{1,s}  \right\} } \right]-1 
                    \le^{(b)}  \frac{1}{ \frac{1}{\sqrt{2\pi}} \cdot \frac{1}{2} \cdot e^{-0.5}} -1 
                    \le  
                    12.176,
    \end{array}
\end{equation}
where step (a) uses $\mu_1  \le \hat{\mu}_{1,s} + \ln^{\alpha}(T)$ and step (b) uses the anti-concentration bound shown in (\ref{Fact 1}).

For any $s \ge 2$, since $\hat{\mu}_{1,s}$ is a random variable in $[0,1]$, we know $\left|\hat{\mu}_{1, s} - \mu_1 \right| \in [0,1]$ is also a random variable. Now, we define a sequence of disjoint sub-intervals 
\begin{equation*}
    \begin{array}{l}
         \left[0,  \sqrt{\frac{2\ln(2)}{s}} \right),  \left[\sqrt{\frac{2\ln(2)}{s}}, \sqrt{\frac{2\ln(2+1)}{s}} \right), \dotsc, \left[\sqrt{\frac{2\ln(r+1)}{s}}, \sqrt{\frac{2\ln(r+2)}{s}} \right), \dotsc, \left[ \sqrt{\frac{2\ln(r_0(s) +1)}{s}}, \sqrt{\frac{2\ln(r_0(s)+2)}{s}}\right),
    \end{array}
\end{equation*}
where $r_0(s) $ is the smallest integer  such that $[0,1] \subseteq   \left[0,  \sqrt{\frac{2\ln(2)}{s}} \right)  \cup \left(\mathop{\bigcup}\limits_{1 \le r \le r_0(s)} \left[\sqrt{\frac{2\ln(r+1)}{s}}, \sqrt{\frac{2\ln(r+2)}{s}} \right)\right)$.

We also define events $\mathcal{S}_0 :=  \left\{\left|\hat{\mu}_{1, s} - \mu_1 \right| \in \left[0,  \sqrt{\frac{2\ln(2)}{s}} \right) \right\}$ and 
 $\mathcal{S}_r :=  \left\{\left|\hat{\mu}_{1, s} - \mu_1 \right| \in \left[\sqrt{\frac{2\ln(r+1)}{s}},  \sqrt{\frac{2\ln(r+2)}{s}} \right) \right\}$ for all $1 \le r \le r_0(s)$ accordingly.

Now, we have
\begin{equation}
    \begin{array}{l}
            \text{LHS of (\ref{WWW1})} =  \mathbb{E}\left[\frac{1}{\mathbb{P} \left\{ \theta_{1,s} > \mu_1 \mid \hat{\mu}_{1,s}  \right\} } \right]-1    

          \le   \mathbb{E}\left[\frac{\bm{1} \left\{ \mathcal{S}_0\right\}}{\mathbb{P} \left\{ \theta_{1,s} > \mu_1 \mid \hat{\mu}_{1,s}  \right\} } \right] + \sum\limits_{1 \le r \le r_0(s)} \mathbb{E}\left[\frac{\bm{1} \left\{ \mathcal{S}_r\right\}}{\mathbb{P} \left\{ \theta_{1,s} > \mu_1 \mid \hat{\mu}_{1,s}  \right\} } \right]-1 \quad.
    \end{array}
    \label{impala 1}
\end{equation}

For the first term in (\ref{impala 1}), we have
\begin{equation}
    \begin{array}{ll}
          &\mathbb{E}\left[\frac{\bm{1} \left\{ \mathcal{S}_0\right\}}{\mathbb{P} \left\{ \theta_{1,s} > \mu_1 \mid \hat{\mu}_{1,s}  \right\} } \right]   \le      \mathbb{E}\left[\frac{\bm{1} \left\{ \mathcal{S}_0\right\}}{\mathbb{P} \left\{ \theta_{1,s} > \hat{\mu}_{1,s} + \sqrt{\frac{2\ln (2)}{s}}\mid \hat{\mu}_{1,s}  \right\} } \right]
    \le      \mathbb{E}\left[\frac{\bm{1} \left\{ \mathcal{S}_0\right\}}{\mathbb{P} \left\{ \theta_{1,s} > \hat{\mu}_{1,s} + \sqrt{\frac{2\ln (2)\ln^{\alpha}(T)}{s}}\mid \hat{\mu}_{1,s}  \right\} } \right] \\

    \le & \frac{1}{ \frac{1}{\sqrt{2\pi}}  \cdot \frac{\sqrt{2\ln (2)} }{2\ln(2)+1} \cdot e^{-0.5 \cdot 2 \cdot \ln(2)}  } 
  \le  10.161,
    \end{array}
    \label{disco1}
\end{equation}
where the second last inequality uses the anti-concentration bound shown in (\ref{Fact 1}).

\

For the second term  in (\ref{impala 1}), we have
\begin{equation}
    \begin{array}{ll}
      &    \sum\limits_{1 \le r \le r_0(s)} \mathbb{E}\left[\frac{\bm{1} \left\{ \mathcal{S}_r\right\}}{\mathbb{P} \left\{ \theta_{1,s} > \mu_1 \mid \hat{\mu}_{1,s}  \right\} } \right] \\

  \le     &\sum\limits_{1 \le r \le r_0(s)}  \mathbb{E}\left[\frac{\bm{1} \left\{\left|\hat{\mu}_{1, s} - \mu_1 \right| \in \left[\sqrt{\frac{2\ln(r+1)}{s}},  \sqrt{\frac{2\ln(r+2)}{s}} \right) 
 \right\}}{\mathbb{P} \left\{ \theta_{1,s} > \mu_1 \mid \hat{\mu}_{1,s}  \right\} } \right] \\

& \\

 = &\sum\limits_{1 \le r \le r_0(s)}  \mathbb{E}\left[\frac{\bm{1} \left\{\left|\hat{\mu}_{1, s} - \mu_1 \right| \in \left[\sqrt{\frac{2\ln(r+1)}{s}},  \sqrt{\frac{2\ln(r+2)}{s}} \right) 
 \right\}}{\mathbb{P} \left\{ \theta_{1,s} > \hat{\mu}_{1,s}  + \sqrt{\frac{2\ln(r+2)}{s}}\mid \hat{\mu}_{1,s}  \right\} } \right] {\color{red} }
 \\ 
 & \\
  \le &\sum\limits_{1 \le r \le r_0(s)}  \mathbb{E}\left[\frac{\bm{1} \left\{\left|\hat{\mu}_{1, s} - \mu_1 \right| \in \left[\sqrt{\frac{2\ln(r+1)}{s}},  \sqrt{\frac{2\ln(r+2)}{s}} \right) 
 \right\}}{\mathbb{P} \left\{ \theta_{1,s} > \hat{\mu}_{1,s}  + \sqrt{\frac{2\ln(r+2) \ln^{\alpha}(T)}{s}}\mid \hat{\mu}_{1,s}  \right\} } \right] 
 \\ 

&\\

\le^{(a)} &\sum\limits_{1 \le r \le r_0(s)}  \mathbb{E}\left[    \frac{1}{\frac{1}{\sqrt{2\pi}} \cdot \frac{\sqrt{2\ln(r+2)}}{2\ln(r+2)+1} \cdot e^{-0.5 \cdot 2 \ln(r+2)}}  \cdot \bm{1} \left\{\left|\hat{\mu}_{1, s} - \mu_1 \right| \in \left[\sqrt{\frac{2\ln(r+1)}{s}},  \sqrt{\frac{2\ln(r+2)}{s}} \right) 
 \right\} \right] \\ 

& \\
%& \le &\sum\limits_{r \ge 1}  \frac{\sqrt{2\pi} (2\ln(r+2)+1) }{\sqrt{2\ln(r+2)} \cdot e^{-\ln(r+2)}} \cdot \mathbb{E}\left[ \bm{1} \left\{\left|\hat{\mu}_{1, s} - \mu_1 \right| \ge \sqrt{\frac{2\ln(r+1)}{s}} \right\}  \right] \\
  =  & \sum\limits_{1 \le r \le r_0(s)} \frac{\sqrt{2\pi} (2\ln(r+2)+1) }{\sqrt{2\ln(r+2)} \cdot e^{-\ln(r+2)}} \cdot  \mathbb{P}\left\{\left|\hat{\mu}_{1, s} - \mu_1 \right| \ge \sqrt{\frac{2\ln(r+1)}{s}}
 \right\}   \\

& \\
 \le^{(b)} & \sum\limits_{1 \le r \le r_0(s)} \frac{\sqrt{2\pi} (2\ln(r+2)+1) }{\sqrt{2\ln(r+2)} \cdot e^{-\ln(r+2)}} \cdot 2 e^{-2 s \cdot \frac{2\ln(r+1)}{s}} \\
 & \\
= & \sum\limits_{1 \le r \le r_0(s)}  \frac{\sqrt{\pi} \cdot  (2\ln(r+2)+1)  \cdot (r+2)}{\sqrt{\ln(r+2)} } \cdot 2 \frac{1}{(r+1)^4} \\
& \\
\le & 3.176  \quad,
    \end{array}
    \label{disco2}
\end{equation}
where step (a) uses the anti-concentration bound shown in (\ref{Fact 1}) and step (b) uses Hoeffding's inequality.

Plugging the results shown in (\ref{disco1}) and (\ref{disco2}) into (\ref{impala 1}), we have
\begin{equation}
    \begin{array}{l}
         \text{LHS of (\ref{WWW1})} =    \mathbb{E}\left[\frac{1}{\mathbb{P} \left\{ \theta_{1,s} > \mu_1 \mid \hat{\mu}_{1,s}  \right\} } \right]-1    

          \le   \mathbb{E}\left[\frac{\bm{1} \left\{ \mathcal{S}_0\right\}}{\mathbb{P} \left\{ \theta_{1,s} > \mu_1 \mid \hat{\mu}_{1,s}  \right\} } \right] + \sum\limits_{r \ge 1} \mathbb{E}\left[\frac{\bm{1} \left\{ \mathcal{S}_r\right\}}{\mathbb{P} \left\{ \theta_{1,s} > \mu_1 \mid \hat{\mu}_{1,s}  \right\} } \right]-1 
\le  12.34,
    \end{array}
\end{equation}
which concludes the proof of the first result. 

\newpage
For the result shown in (\ref{WWW11}),
  we define the following  sequence of sub-intervals 
\begin{equation*}
    \begin{array}{l}
    \left[0,  \sqrt{\frac{\ln(T \Delta_i^2)}{s}} \right), \dotsc, \left[\sqrt{\frac{\ln(r \cdot T \Delta_i^2)}{s}}, \sqrt{\frac{\ln( (r+1) \cdot T \Delta_i^2)}{s}} \right), \dotsc, \left[\sqrt{\frac{\ln(r_0(s) \cdot T \Delta_i^2)}{s}}, \sqrt{\frac{\ln( (r_0(s)+1) \cdot T \Delta_i^2)}{s}} \right),   
    \end{array}
\end{equation*}
where $r_0(s)$ is the smallest integer such that $[0,1] \subseteq      \left[0,  \sqrt{\frac{\ln(T \Delta_i^2)}{s}} \right) \mathop{\bigcup}\limits_{1 \le r \le r_0(s)} \left[\sqrt{\frac{\ln(r \cdot T \Delta_i^2)}{s}}, \sqrt{\frac{\ln( (r+1) \cdot T \Delta_i^2)}{s}} \right)$. % and integer $1 \le r \le r_0(s)$ and %Note that, for any fixed {\color{blue}$s \ge \frac{4(1+\sqrt{2})^2 \ln(T \Delta_i^2)}{\Delta_i^2}$ }, we can always find the smallest integer 

%Note that, for a finite learning horizon $T$, the length of the intervals is finite. 

We define events $\mathcal{S}_0 :=  \left\{\left|\hat{\mu}_{1, s} - \mu_1 \right| \in \left[0,  \sqrt{\frac{\ln(T \Delta_i^2)}{s}} \right) \right\}$ and   $\mathcal{S}_r :=  \left\{\left|\hat{\mu}_{1, s} - \mu_1 \right| \in \left[\sqrt{\frac{\ln(r  T \Delta_i^2)}{s}},  \sqrt{\frac{\ln((r+1)  T \Delta_i^2)}{s}} \right) \right\}$ for all $1 \le r \le r_0(s)$ accordingly.

From $s \ge \frac{4(1+\sqrt{2})^2 \ln(T \Delta_i^2) \ln^{\alpha}(T)}{\Delta_i^2}$, we  also have $\Delta_i \ge  \sqrt{\frac{4 (1+\sqrt{2})^2 \ln(T \Delta_i^2) \ln^{\alpha}(T)}{s}}$. Then, we have
\begin{equation}
    \begin{array}{ll}
       & \text{LHS of (\ref{WWW11})} \\
        = &  \mathbb{E}\left[\frac{1}{\mathbb{P} \left\{ \theta_{1,s} > \mu_1- 0.5\Delta_i \mid \hat{\mu}_{1,s}  \right\} } \right]-1    \\
        & \\
\le & \mathbb{E}\left[\frac{1}{\mathbb{P} \left\{ \theta_{1,s} > \mu_1 - \sqrt{\frac{(1+\sqrt{2})^2\ln \left(T \Delta_i^2 \right) \ln^{\alpha}(T)}{s}}\mid \hat{\mu}_{1,s}  \right\} } \right]-1    \\ 
& \\
          \le  & \left(\mathbb{E}\left[\frac{\bm{1} \left\{ \mathcal{S}_0\right\}}{\mathbb{P} \left\{ \theta_{1,s} > \mu_1 - \sqrt{\frac{(1+\sqrt{2})^2 \ln \left(T \Delta_i^2 \right) \ln^{\alpha}(T)}{s}}\mid \hat{\mu}_{1,s}  \right\} } \right] - 1\right) + \sum\limits_{1 \le r \le r_0(s)} \mathbb{E}\left[\frac{\bm{1} \left\{ \mathcal{S}_r\right\}}{\mathbb{P} \left\{ \theta_{1,s} > \mu_1 - \sqrt{\frac{(1+\sqrt{2})^2 \ln \left(T \Delta_i^2 \right) \ln^{\alpha}(T)}{s}}\mid \hat{\mu}_{1,s}  \right\} } \right]\\
       & \\
         \le  & \left(\mathbb{E}\left[\frac{\bm{1} \left\{ \mathcal{S}_0\right\}}{\mathbb{P} \left\{ \theta_{1,s} > \mu_1 - \sqrt{\frac{(1+\sqrt{2})^2 \ln \left(T \Delta_i^2 \right) \ln^{\alpha}(T)}{s}}\mid \hat{\mu}_{1,s}  \right\} } \right] - 1\right) + \sum\limits_{1 \le r \le r_0(s)} \mathbb{E}\left[\frac{\bm{1} \left\{ \mathcal{S}_r\right\}}{\mathbb{P} \left\{ \theta_{1,s} > \mu_1\mid \hat{\mu}_{1,s}  \right\} } \right] \quad. 
    \end{array}
    \label{mix 1}
\end{equation}

%Now, we upper bound each term separately. 

For the first term in (\ref{mix 1}), we have
\begin{equation}
    \begin{array}{ll}
    &  \mathbb{E}\left[\frac{\bm{1} \left\{ \mathcal{S}_0\right\}}{\mathbb{P} \left\{ \theta_{1,s} > \mu_1 -\frac{1}{2}\sqrt{\frac{(1+\sqrt{2})^2 \ln \left(T \Delta_i^2 \right) \ln^{\alpha}(T)}{s}}\mid \hat{\mu}_{1,s}  \right\} } \right]  -1  \\
    \le     &   \mathbb{E}\left[\frac{\bm{1} \left\{ \mathcal{S}_0\right\}}{\mathbb{P} \left\{ \theta_{1,s} > \hat{\mu}_{1,s} + \sqrt{\frac{ \ln \left(T \Delta_i^2 \right)}{s}} - \sqrt{\frac{ (1+\sqrt{2})^2 \ln \left(T \Delta_i^2 \right)\ln^{\alpha}(T)}{s}}\mid \hat{\mu}_{1,s}  \right\} } \right]-1 \\
     \le     &   \mathbb{E}\left[\frac{\bm{1} \left\{ \mathcal{S}_0\right\}}{\mathbb{P} \left\{ \theta_{1,s} > \hat{\mu}_{1,s} + \sqrt{\frac{ \ln \left(T \Delta_i^2 \right) \ln^{\alpha}(T)}{s}} - \sqrt{\frac{ (1+\sqrt{2})^2 \ln \left(T \Delta_i^2 \right)\ln^{\alpha}(T)}{s}}\mid \hat{\mu}_{1,s}  \right\} } \right]-1 \\
         =    &  \mathbb{E}\left[\frac{\bm{1} \left\{ \mathcal{S}_0\right\}}{\mathbb{P} \left\{ \theta_{1,s} > \hat{\mu}_{1,s}  - \sqrt{\frac{  2\ln(T \Delta_i^2)\ln^{\alpha}(T)}{s}}\mid \hat{\mu}_{1,s}  \right\} } \right] -1\\
       \le^{(a)}    &   \mathbb{E}\left[\frac{1}{1-\frac{0.5}{T\Delta_i^2}} \right] -1\\  
%= & \frac{1}{1- \frac{1}{(T\Delta_i^2)^{0.125 \times 8}}} -1 \\
        \le^{(b)} & \frac{0.613}{T \Delta_i^2} \quad,
    %   = & O(\frac{1}{T\Delta_i^2})\quad,
    %   = & \frac{1 - e^{-0.5 \blacksquare \ln(T \Delta_i^2) /4} + e^{-0.5 \blacksquare \ln(T \Delta_i^2) /4}}{1- e^{-0.5 \blacksquare \ln(T \Delta_i^2) /4}}  \\
     %  = & 1 + \frac{e^{-0.5 \blacksquare \ln(T \Delta_i^2) /4}}{1-e^{-0.5 \blacksquare \ln(T \Delta_i^2) /4}} \\
    %   \le & 1 + C \cdot e^{-0.5 \blacksquare \ln(T \Delta_i^2) /4} 
    \end{array}
\end{equation}
where step (a) uses concentration bound shown in (\ref{Fact 2}) %, i.e., $ \mathbb{P} \left\{ \theta_{1,s} > \hat{\mu}_{1,s}  -\sqrt{\frac{2 \ln(T \Delta_i^2)}{s}}\mid \hat{\mu}_{1,s}  \right\}   \ge  1- e^{-0.5 \cdot 2 \ln(T \Delta_i^2)} = 1-\frac{1}{T\Delta_i^2}$, 
and step (b) uses $\frac{1}{1-\frac{0.5}{T\Delta_i^2}} -1 = \frac{\frac{0.5}{T\Delta_i^2}}{1-\frac{0.5}{T\Delta_i^2}} \le  \frac{0.5}{T\Delta_i^2} \cdot \frac{1}{1-0.5/e}$. %Note that we only consider the mean reward gap such that $T \Delta_i^2 >2$.

For the second term in (\ref{mix 1}), we have
\begin{equation}
    \begin{array}{ll}
         & \sum\limits_{1 \le r \le r_0(s)} \mathbb{E}\left[\frac{\bm{1} \left\{ \mathcal{S}_r\right\}}{\mathbb{P} \left\{ \theta_{1,s} > \mu_1 \mid \hat{\mu}_{1,s}  \right\} } \right] \\
=   &  \sum\limits_{1 \le r \le r_0(s)} \mathbb{E}\left[\frac{\bm{1} \left\{ \left|\hat{\mu}_{1, s} - \mu_1 \right| \in \left[\sqrt{\frac{\ln \left(r \cdot T \Delta_i^2 \right)}{s}},  \sqrt{\frac{\ln \left((r+1) \cdot T \Delta_i^2 \right)}{s}} \right) \right\}}{\mathbb{P} \left\{ \theta_{1,s} > \mu_1 \mid \hat{\mu}_{1,s}  \right\} } \right] \\
    
                \le   &  \sum\limits_{1 \le r \le r_0(s)} \mathbb{E}\left[\frac{\bm{1} \left\{ \left|\hat{\mu}_{1, s} - \mu_1 \right| \in \left[\sqrt{\frac{\ln \left(r \cdot T \Delta_i^2 \right)}{s}},  \sqrt{\frac{\ln \left((r+1) \cdot T \Delta_i^2 \right)}{s}} \right) \right\}}{\mathbb{P} \left\{ \theta_{1,s} > \hat{\mu}_{1,s}  + \sqrt{\frac{\ln \left((r+1) \cdot T \Delta_i^2 \right)}{s}}\mid \hat{\mu}_{1,s}  \right\} } \right] \\
                    \le   &  \sum\limits_{1 \le r \le r_0(s)} \mathbb{E}\left[\frac{\bm{1} \left\{ \left|\hat{\mu}_{1, s} - \mu_1 \right| \in \left[\sqrt{\frac{\ln \left(r \cdot T \Delta_i^2 \right)}{s}},  \sqrt{\frac{\ln \left((r+1) \cdot T \Delta_i^2 \right)}{s}} \right) \right\}}{\mathbb{P} \left\{ \theta_{1,s} > \hat{\mu}_{1,s}  + \sqrt{\frac{\ln \left((r+1) \cdot T \Delta_i^2 \right) \ln^{\alpha}(T)}{s}}\mid \hat{\mu}_{1,s}  \right\} } \right] \\
      \le^{(a)} &  \sum\limits_{1 \le r \le r_0(s)} \frac{1}{\frac{1}{2\sqrt{2\pi}} \cdot \frac{1}{\sqrt{\ln((r+1) \cdot T \Delta_i^2)}} \cdot ((r+1) \cdot T \Delta_i^2)^{-0.5}  } \cdot \mathbb{P} \left\{ \left|\hat{\mu}_{1, s} - \mu_1 \right| \in \left[\sqrt{\frac{\ln(r \cdot T \Delta_i^2)}{s}},  \sqrt{\frac{\ln((r+1) \cdot T \Delta_i^2)}{s}} \right) \right\} \\
            \le & \sum\limits_{1 \le r \le r_0(s)}  \frac{1}{ \frac{1}{2\sqrt{2\pi}} \cdot \frac{1}{\sqrt{\ln((r+1) \cdot T \Delta_i^2)}} \cdot ((r+1) \cdot T \Delta_i^2)^{-0.5} } \cdot \mathbb{P} \left\{ \left|\hat{\mu}_{1, s} - \mu_1 \right| \ge \sqrt{\frac{\ln(r \cdot T \Delta_i^2)}{s}} \right\} \\
             \le^{(b)} &  \sum\limits_{1 \le r \le r_0(s)} \frac{1}{\frac{1}{2\sqrt{2\pi}} \cdot \frac{1}{\sqrt{\ln((r+1) \cdot T \Delta_i^2)}} \cdot ((r+1) \cdot T \Delta_i^2)^{-0.5}  } \cdot 2e^{-2 \ln \left(r \cdot T \Delta_i^2 \right)}\\
             & \\
             = & \sum\limits_{1 \le r \le r_0(s)}  \frac{1}{\frac{1}{2\sqrt{2\pi}} \cdot \frac{1}{\sqrt{\ln((r+1) \cdot T \Delta_i^2)}} \cdot ((r+1) \cdot T \Delta_i^2)^{-0.5} } \cdot 2 (r \cdot T \Delta_i^2)^{-2}\\
             & \\
             \le & \sum\limits_{1 \le r \le r_0(s)} \frac{4\sqrt{2\pi (r+1) \cdot T \Delta_i^2 \cdot  \ln \left( (r+1) \cdot T \Delta_i^2 \right)} }{\left(r \cdot T\Delta_i^2 \right)^2} \\

              & \\
             = &\frac{4\sqrt{2\pi }}{T \Delta_i^2} \sum\limits_{1 \le r \le r_0(s)} \frac{\sqrt{ (r+1) \cdot   \ln \left( (r+1) \cdot T \Delta_i^2 \right)} }{r^2 \cdot \sqrt{T\Delta_i^2}} \\
             \leq& \frac{4\sqrt{2\pi}}{T \Delta_i ^2}\sum\limits_{1 \le r \le r_0(s)} \frac{\sqrt{(r+1)\ln (r+1)}}{r^2} + \frac{\sqrt{r+1}}{r^2} \\
             \leq & \frac{4\sqrt{2\pi}}{T \Delta_i ^2} \times 7.034\\
             & \\
             \le &   \frac{70.5235}{T \Delta_i ^2}\quad,
             % = & O(\frac{1}{T\Delta_i^2})\quad,
    \end{array}
\end{equation}
where step (a) uses the anti-concentration bound shown in (\ref{Fact 1}), i.e., we have
\begin{equation*}
    \begin{array}{lll}
          \mathbb{P} \left\{ \theta_{1,s} > \hat{\mu}_{1,s}  + \sqrt{\frac{\ln((r+1) \cdot T \Delta_i^2)}{s}}\mid \hat{\mu}_{1,s}  \right\} & \ge &  \mathbb{P} \left\{ \theta_{1,s} > \hat{\mu}_{1,s}  + \sqrt{\frac{\ln((r+1) \cdot T \Delta_i^2) \ln^{\alpha}(T)}{s}}\mid \hat{\mu}_{1,s}  \right\} \\
     &  \ge &  \frac{1}{\sqrt{2\pi}} \cdot \frac{\sqrt{\ln((r+1) \cdot T \Delta_i^2)}}{\ln((r+1) \cdot T \Delta_i^2) + 1} \cdot e^{-0.5 \cdot \ln((r+1) \cdot T \Delta_i^2)}  \\
     &  =  & \frac{1}{\sqrt{2\pi}} \cdot \frac{\sqrt{\ln((r+1) \cdot T \Delta_i^2)}}{\ln((r+1) \cdot T \Delta_i^2) + 1} \cdot ((r+1) \cdot T \Delta_i^2)^{-0.5} \\
     &  >  & \frac{1}{\sqrt{2\pi}} \cdot \frac{\sqrt{\ln((r+1) \cdot T \Delta_i^2)}}{2\ln((r+1) \cdot T \Delta_i^2) } \cdot ((r+1) \cdot T \Delta_i^2)^{-0.5} \\
      &   =  & \frac{1}{2\sqrt{2\pi}} \cdot \frac{1}{\sqrt{\ln((r+1) \cdot T \Delta_i^2)}} \cdot ((r+1) \cdot T \Delta_i^2)^{-0.5} \quad.
         %\ge & \frac{1}{2\sqrt{2\pi}} \cdot \frac{1}{\sqrt{\ln((r+1) \cdot T \Delta_i^2)}} \cdot ((r+1) \cdot T \Delta_i^2)^{-0.5} \\
    \end{array}
\end{equation*}\end{proof}

\newpage
\section{Proofs for Theorem~\ref{thm: regret}} \label{app: regret proof}

% {\color{blue}
% Recall $\theta_{i}(t)  \sim \mathcal{N}\left(\hat{\mu}_{i,n_i}, \ \frac{ \ln^{\alpha}(T)}{n_{i}} {\color{red}+ \frac{ \ln(1.25/\delta_0)}{\varepsilon_0^2 n_i^2}}\right)$ and $\phi = 2 (T)^{0.5(1-\alpha)} \ln^{0.5(3-\alpha)}(T) /c_0$.}

%{\color{red}Assume $T \Delta_i^2 >e$.}
\begin{proof}We first define two high-probability events. For any arm $i \in [K]$, let $\mathcal{E}_i^{\mu}(t-1) := \left\{\left| \hat{\mu}_{i,n_i(t-1)} - \mu_i \right| \le \sqrt{\frac{\ln(T \Delta_i^2)}{n_i(t-1)}} \right\}$ and  $\mathcal{E}_i^{\theta}(t) := \left\{\theta_i(t) \le \hat{\mu}_{i,n_i(t-1)} + \sqrt{2\ln(T\Delta_i^2 \cdot \phi)} \cdot\sqrt{\frac{ \ln^{\alpha}(T)}{n_i(t-1)} } \right\}$. Let $\overline{\mathcal{E}_i^{\mu}(t-1)}$ and $\overline{\mathcal{E}_i^{\theta}(t)}$ denote the complements, respectively.

Fix a sub-optimal arm $i$. Let $L_i := \frac{ \left(\sqrt{2}+1 \right)^2}{4} \cdot \ln(T \Delta_i^2 \cdot \phi) \cdot \frac{\ln^{\alpha}(T)}{\Delta_i^2}$ and
 $r_i^{(*)} = \left\lceil \log_2(L_i) \right\rceil$.
    
 Let $t_0$ denote the last round of epoch $r_i^{(*)}$. That is also to say, at the end of round $t_0$, arm $i$'s empirical mean will be updated by using $2^{r_i^{(*)}}$ observations.

    Let $N_i(T)$ denote the number of pulls of sub-optimal arm $i$ by the end of round $T$.  We upper bound $ \mathbb{E} \left[N_i(T) \right] $, the expected number of pulls of sub-optimal arm $i$. We decompose the regret based on whether the above-defined events are true or not. We have
\begin{equation}
    \begin{array}{lll}
        \mathbb{E} \left[N_i(T) \right] 
         
      &   = &

      \sum_{t=K+1}^{T} \mathbb{E} \left[\bm{1} \left\{i_t = i \right\} \right] + 1 \\
         &   = & 
    \mathbb{E} \left[ \sum_{t=K+1}^{t_0} \bm{1} \left\{i_t = i, n_i(t-1) < L_i \right\} \right] +  \mathbb{E} \left[ \sum_{t=t_0+1}^{T} \bm{1} \left\{i_t = i, n_i(t-1) \ge L_i \right\} \right]+ 1 \\
         & \le & \sum_{s = 1}^{r_i^{(*)}}  2^s + \sum_{t=K+1}^{T} \mathbb{E} \left[\bm{1} \left\{i_t = i, n_i(t-1) \ge L_i \right\} \right] + 1 \\
          & = & \sum_{s = 0}^{r_i^{(*)}}  2^s + \sum_{t=K+1}^{T} \mathbb{E} \left[\bm{1} \left\{i_t = i, n_i(t-1) \ge L_i \right\} \right]  \\
          & \le & 
          4 L_i + \underbrace{\sum_{t=K+1}^{T} \mathbb{E} \left[\bm{1} \left\{i_t = i, \mathcal{E}_i^{\theta}(t), \mathcal{E}_i^{\mu}(t-1),  n_i(t-1) \ge L_i \right\} \right] }_{\omega_1}  \\
         
         &  + & \underbrace{ \sum_{t=K+1}^{T} \mathbb{E} \left[\bm{1} \left\{i_t = i,  \overline{\mathcal{E}_i^{\theta}(t)}, n_i(t-1) \ge L_i \right\} \right]}_{\omega_2 = O(1/\Delta_i^2), \text{ Lemma~\ref{lemma: theta}}} 
           + \underbrace{  \sum_{t=K+1}^{T} \mathbb{E} \left[\bm{1} \left\{i_t = i, \overline{\mathcal{E}_i^{\mu}(t-1)}, n_i(t-1) \ge L_i \right\} \right]}_{\omega_3 = O(1/\Delta_i^2), \text{ Lemma~\ref{lemma: hat}}}. 
           \end{array}
           \label{Friday 1}
           \end{equation} 
For $\omega_2$ and $\omega_3$ terms, we prepare a lemma for each of them.

\begin{lemma}\label{lemma: theta}
    We have $\sum_{t=K+1}^{T} \mathbb{E} \left[\bm{1} \left\{i_t = i,  \overline{\mathcal{E}_i^{\theta}(t)}, n_i(t-1) \ge L_i \right\} \right]  \le O \left(  \frac{1}{\Delta_i^2} \right)$.
         \end{lemma}
         \begin{lemma}
\label{lemma: hat}
    We have $      \sum_{t=K+1}^{T} \mathbb{E} \left[\bm{1} \left\{i_t = i, \overline{\mathcal{E}_i^{\mu}(t-1)}, n_i(t-1) \ge L_i \right\} \right]  \le  O \left(  \frac{1}{\Delta_i^2} \right)$.
\end{lemma}

 The challenging part is to upper bound term $\omega_1$. 
%However, the regret for this term does not involve privacy parameters $\varepsilon_0, \delta_0$. 
By tuning $L_i$ properly, we have 
\begin{equation}
    \begin{array}{lll}
       \omega_1 & = & \sum_{t=K+1}^{T} \mathbb{E} \left[\bm{1} \left\{i_t = i, \mathcal{E}_i^{\theta}(t), \mathcal{E}_i^{\mu}(t-1),  n_i(t-1) \ge L_i \right\} \right]  \\
       & \le^{(a)} & \sum_{t=K+1}^{T} \mathbb{E} \left[\bm{1} \left\{i_t = i, \theta_i(t) \le \mu_i + 0.5 \Delta_i \right\} \right]\quad,     
    \end{array}
\end{equation}
where step (a) uses the argument that if both events  $\mathcal{E}_i^{\mu}(t-1) = \left\{\left| \hat{\mu}_{i,n_i(t-1)} - \mu_i \right| \le \sqrt{\frac{\ln(T \Delta_i^2)}{n_i(t-1)}} \right\}$ and  $\mathcal{E}_i^{\theta}(t) = \left\{\theta_i(t) \le \hat{\mu}_{i,n_i(t-1)} + \sqrt{2\ln(T\Delta_i^2 \cdot \phi)} \cdot\sqrt{\frac{ \ln^{\alpha}(T)}{n_i(t-1)} } \right\}$ are true, and $n_i(t-1) \ge L_i$, we have 
\begin{equation}
    \begin{array}{lll}
       \theta_i(t)  & \le & \hat{\mu}_{i,n_i(t-1)} + \sqrt{2\ln(T\Delta_i^2 \cdot \phi)} \cdot\sqrt{\frac{ \ln^{\alpha}(T)}{n_i(t-1)} }  \\
       & \le & \mu_i + \sqrt{\frac{\ln(T \Delta_i^2)}{n_i(t-1)}}  + \sqrt{2\ln(T\Delta_i^2 \cdot \phi)} \cdot\sqrt{\frac{ \ln^{\alpha}(T)}{n_i(t-1)} }  \\
       & < & \mu_i + \sqrt{\frac{\ln(T \Delta_i^2 \cdot \phi)}{L_i}} \sqrt{\ln^{\alpha}(T)} + \sqrt{2\ln(T\Delta_i^2 \cdot \phi)} \cdot\sqrt{\frac{ \ln^{\alpha}(T)}{L_i} }  \\
       & < & \mu_i + (\sqrt{2}+1)\sqrt{\frac{\ln(T \Delta_i^2 \cdot \phi)}{L_i}} \sqrt{\ln^{\alpha}(T)}\\
       & = & \mu_i + 0.5\Delta_i\quad,
    \end{array}
\end{equation}
where the last step applies $L_i =\frac{ \left(\sqrt{2}+1\right)^2}{4} \cdot \ln(T \Delta_i^2 \cdot \phi) \cdot \frac{\ln^{\alpha}(T)}{\Delta_i^2} $. % we have $\mathcal{E}_i^{\theta}(t)  \le \mu_i + 0.5\Delta_i$.

Since the optimal arm $1$ can either be in the mandatory TS-Gaussian phase or the optional UCB phase, 
%we know $\theta_i(t)$ can either be a fresh Gaussian sample or {\color{red}the maximum of $\phi$ i.i.d. Gaussian samples}. 
we continue decomposing the regret based on the case of the optimal arm $1$. 
%that is, whether the optimal arm $1$ is in Phase I or Phase II. 
Define $\mathcal{T}_1(t)$ as the event that the optimal arm $1$ in round $t$ is in the mandatory TS-Gaussian phase, that is, using a fresh Gaussian mean reward model in the learning. Let $\overline{\mathcal{T}_1(t)}$ denote the complement, that is, using $\text{MAX}_1 = \mathop{\max}_{h_1 \in [\phi]} \theta^{(h_1)}_{1, n_1(t-1)}$ in the learning, where all $\theta^{(h_1)}_{1, n_1(t-1)}  \sim \mathcal{N}\left(\hat{\mu}_{1,n_1(t-1)}, \frac{\ln^{\alpha}(T)}{n_{1}(t-1)} \right)$ are i.i.d. random variables. 

We have
\begin{equation}
    \begin{array}{lll}
\omega_1 & \le & \underbrace{ \sum\limits_{t=K+1}^{T} \mathbb{E} \left[\bm{1} \left\{i_t = i, \theta_i(t) \le \mu_i + 0.5 \Delta_i , \mathcal{T}_1(t) \right\} \right]}_{I_1} + \underbrace{\sum\limits_{t=K+1}^{T} \mathbb{E} \left[\bm{1} \left\{i_t = i, \theta_i(t) \le \mu_i + 0.5 \Delta_i , \overline{\mathcal{T}_1(t)} \right\} \right]}_{I_2}.
 \end{array}
\end{equation}

\paragraph{Upper bound $I_1$.} Note that if event $\mathcal{T}_1(t)$ is true, we know  the optimal arm $1$ is using  a fresh Gaussian mean reward model in the learning in round $t$, that is, $\theta_1(t) \sim \mathcal{N}\left(\hat{\mu}_{1,n_1(t-1)}, \frac{\ln^{\alpha}(T)}{n_{1}(t-1)} \right)$.
%{\color{red}The allowing of drawing a random sample  means the optimal arm $1$ will have some chance to be pulled: weak optimism instead of strong optimism...}
Term $I_1$  will use a similar analysis to Lemma~2.8 of \citet{agrawalnear}, which links the probability of pulling a sub-optimal arm $i$ to the probability of pulling the optimal arm $1$.   We formalize this into our technical Lemma~\ref{martingale lemma} below. 
Let $\mathcal{F}_{t-1} = \left\{h_1(\tau), h_2(\tau), \dotsc, h_K(\tau), i_{\tau}, X_{i_{\tau}}(\tau), \forall \tau =1,2,\dotsc,t-1 \right\}$ collect all the history information by the end of round $t-1$. It collects the number of unused  Gaussian sampling budget $h_i(\tau)$ by the end of round $\tau$ for all $i \in [K]$, the index $i_{\tau}$ of the pulled arm, and the observed reward $X_{i_{\tau}}(\tau)$ for all rounds $\tau = 1, 2,\dotsc,t-1$. Let $\theta_{1,n_1(t-1)} \sim \mathcal{N} \left(\hat{\mu}_{1,n_1(t-1)}, \ \frac{\ln^{\alpha}(T)}{n_1(t-1)} \right)$ be a Gaussian random variable. %Note that the variance does not include $\varepsilon_0$ and $\delta_0$.

%Note that detailed information about posterior samples is not collected. 
\begin{lemma}
        For any instantiation $F_{t-1}$ of $\mathcal{F}_{t-1}$, we have 
           \begin{equation}
        \begin{array}{ll}
           &  \mathbb{E} \left[   \bm{1} \left\{i_t = i,\mathcal{T}_1(t) , \theta_i(t) \le \mu_i + 0.5\Delta_i \right\} \mid \mathcal{F}_{t-1} = F_{t-1}\right]  \\

             \le & \left(\frac{1 }{ \mathbb{P} \left\{\theta_{1,n_1(t-1)} > \mu_1 - 0.5\Delta_i \mid \mathcal{F}_{t-1} = F_{t-1} \right\} } -1\right)   \mathbb{E} \left[   \bm{1} \left\{i_t = 1\right\} \mid \mathcal{F}_{t-1} = F_{t-1}\right].
        \end{array}
       \end{equation}
       \label{martingale lemma}
       \end{lemma}
With Lemma~\ref{martingale lemma} in hand, we upper bound term $I_1$. 
  Let $ L_{1,i} := \frac{4(1+\sqrt{2})^2 \ln(T \Delta_i^2) \ln^{\alpha}(T)}{\Delta_i^2}$. Let $r_1^{(*)} = \left\lceil \log_2(L_{1,i})\right\rceil$. We have 
\begin{equation}
    \begin{array}{lll}
         I_1 & = &   \sum\limits_{t=K+1}^{T} \mathbb{E} \left[\bm{1} \left\{i_t = i, \theta_i(t) \le \mu_i + 0.5 \Delta_i , \mathcal{T}_1(t) \right\} \right] \\
         & = & \sum\limits_{t=K+1}^{T} \mathbb{E} \left[\mathbb{E} \left[\bm{1} \left\{i_t = i, \theta_i(t) \le \mu_i + 0.5 \Delta_i , \mathcal{T}_1(t) \right\} \mid \mathcal{F}_{t-1}\right] \right]  \\
         & \le & \sum\limits_{t=K+1}^{T} \mathbb{E} \left[   \left(\frac{1 }{ \mathbb{P} \left\{\theta_{1,n_1(t-1)} > \mu_1 - 0.5\Delta_i \mid \mathcal{F}_{t-1} = F_{t-1} \right\} } -1\right) \cdot   \mathbb{E} \left[   \bm{1} \left\{i_t = 1\right\} \mid \mathcal{F}_{t-1} = F_{t-1}\right]   \right]  \\
          & = & \sum\limits_{t=K+1}^{T} \mathbb{E} \left[\mathbb{E} \left[   \left(\frac{1}{ \mathbb{P} \left\{\theta_{1,n_1(t-1)} > \mu_1 - 0.5\Delta_i \mid \mathcal{F}_{t-1} = F_{t-1} \right\} } -1\right) \cdot      \bm{1} \left\{i_t = 1\right\} \mid \mathcal{F}_{t-1} = F_{t-1}\right]   \right]  \\
           & = & \sum\limits_{t=K+1}^{T} \mathbb{E} \left[   \left(\frac{1 }{ \mathbb{P} \left\{\theta_{1,n_1(t-1)} > \mu_1 - 0.5\Delta_i \mid \mathcal{F}_{t-1} = F_{t-1} \right\} } -1\right) \cdot      \bm{1} \left\{i_t = 1\right\}   \right]  \\
               & \le & \sum\limits_{s=0}^{\log(T)}2^{s+1} \cdot \mathbb{E} \left[   \left(\frac{1}{ \mathbb{P} \left\{\theta_{1,2^s} > \mu_1 - 0.5\Delta_i \mid \hat{\mu}_{1,2^s} \right\} } -1\right)   \right]   \\
               & \le & \sum\limits_{s=0}^{r_1^{(*)}-1}2^{s+1} \cdot \underbrace{\mathbb{E} \left[   \left(\frac{1}{ \mathbb{P} \left\{\theta_{1,2^s} > \mu_1 - 0.5\Delta_i \mid \hat{\mu}_{1,2^s} \right\} } -1\right)   \right] }_{\le 12.34 \text{ from (\ref{WWW1})}}+ \sum\limits_{s=r_1^{(*)}}^{\log(T)}2^{s+1} \cdot \underbrace{ \mathbb{E} \left[   \left(\frac{1}{ \mathbb{P} \left\{\theta_{1,2^s} > \mu_1 - 0.5\Delta_i \mid \hat{\mu}_{1,2^s} \right\} } -1\right)   \right] }_{\le  \frac{72}{T\Delta_i^2} \text{ from (\ref{WWW11})}}  \\
               & \le & 4 \cdot  L_{1,i} \cdot 12.34 + \sum\limits_{s=r_1^{(*)}}^{\log(T)}2^{s+1} \cdot \frac{72}{T\Delta_i^2} \\
%               & \le & {\color{red}C_4} \cdot  L_1 + {\color{red}\sum\limits_{s=r_1^{(*)}}^{\log(T)}2^{s+1} \cdot \mathbb{E} \left[   \left(\frac{1 }{ \mathbb{P} \left\{\theta_{1,2^s} > \mu_1 - 0.5\Delta_i \mid \hat{\mu}_{1,2^s} \right\} } -1\right)   \right]} \\

               & \le & 50 L_{1,i}+ O(1/\Delta_i^2) \quad.
    \end{array}
    \label{I_1}
\end{equation}

%{\color{red}Let $\tau_s^{(1)}$ be the round in which the XXXXXXXXXXX ... }

\paragraph{Upper bound $I_2$.}
    
Note that if event $\mathcal{T}_1(t)$ is false, we know the optimal arm $1$ is using $\text{MAX}_1 = \mathop{\max}_{h_1 \in [\phi]} \theta^{(h_1)}_{1, n_1(t-1)}$ in the learning, where $\theta^{(h_1)}_{1, n_1(t-1)}  \sim \mathcal{N}\left(\hat{\mu}_{1,n_1(t-1)}, \frac{\ln^{\alpha}(T)}{n_{1}(t-1)} \right)$ for each $h_1 \in [\phi]$. We have 
\begin{equation}
    \begin{array}{lll}
         I_2 & = &  \sum\limits_{t=K+1}^{T} \mathbb{E} \left[\bm{1} \left\{i_t = i, \theta_i(t) \le \mu_i + 0.5 \Delta_i , \overline{\mathcal{T}_1(t)} \right\} \right] \\
         & < &  \sum\limits_{t=K+1}^{T} \mathbb{E} \left[\bm{1} \left\{i_t = i, \theta_i(t) \le \mu_i + \Delta_i , \overline{\mathcal{T}_1(t)} \right\} \right] \\
         & \le &  \sum\limits_{t=K+1}^{T} \mathbb{E} \left[\bm{1} \left\{i_t = i, \theta_1(t) \le \mu_i +  \Delta_i , \overline{\mathcal{T}_1(t)} \right\} \right] \\
          & \le &  \sum\limits_{t=K+1}^{T}\sum\limits_{s=0}^{\log(T)} \underbrace{ \mathbb{E} \left[\bm{1} \left\{ \mathop{\max}_{h_1 \in [\phi]} \theta^{(h_1)}_{1, 2^s} \le \mu_1 \right\} \right]}_{\text{Lemma~\ref{lemma boost}}} \\
         & \le & \sum\limits_{t=K+1}^{T}\sum\limits_{s=0}^{\log(T)} O(1/T)     \\
          & \le & O(\ln(T))\quad.
        
    \end{array}
    \label{I_2}
\end{equation}

From (\ref{I_1}) and (\ref{I_2}), we have $ \omega_1     \le
          I_1 + I_2 
          \le  50 L_{1,i}+ O(1/\Delta_i^2) + O(\ln(T)) 
          \le  O \left( \frac{\ln(T \Delta_i^2) \ln^{\alpha}(T)}{\Delta_i^2} \right)$,
        which gives
        \begin{equation}
            \begin{array}{l}
                  \mathbb{E} \left[N_i(T) \right] \le  O \left( \frac{\ln(\phi T \Delta_i^2) \ln^{\alpha}(T)}{\Delta_i^2} \right) + O \left( \frac{\ln(T \Delta_i^2) \ln^{\alpha}(T)}{\Delta_i^2} \right) = O \left( \frac{\ln(\phi T \Delta_i^2) \ln^{\alpha}(T)}{\Delta_i^2} \right)\quad.
            \end{array}
        \end{equation}
        Therefore, the problem-dependent regret bound by the end of round $T$ is
        \begin{equation}
            \begin{array}{ll}
              &  \sum_{i \in [K]: \Delta_i >0} \mathbb{E} \left[N_i(T) \right] \cdot \Delta_i \\
              = & \sum_{i \in [K]: \Delta_i >0} O \left( \frac{\ln(\phi T \Delta_i^2) \ln^{\alpha}(T)}{\Delta_i} \right) \\
              = & \sum_{i \in [K]: \Delta_i >0} O \left( \frac{\ln(c_0 T^{0.5(1-\alpha)}  \ln^{0.5(3-\alpha)}(T)  T \Delta_i^2) \ln^{\alpha}(T)}{\Delta_i} \right) \\
              \le &  \sum_{i \in [K]: \Delta_i >0}  O \left( \frac{\ln \left(   T^{0.5(3-\alpha)} \Delta_i^2 \right) \ln^{\alpha}(T)}{\Delta_i} \right) + O \left( \frac{(3-\alpha) \ln \ln(T) \ln^{\alpha}(T)}{\Delta_i} \right)\quad.
            \end{array}
        \end{equation}

        For the proof of worst-case regret bound, we set the critical gap $\Delta_* := \sqrt{K \ln^{1+\alpha}(T)/T}$. The regret from pulling any sub-optimal arms with mean reward gaps no greater than $\Delta_*$ is at most $T \Delta_* = O ( \sqrt{KT \ln^{1+\alpha}(T)})$. The regret from pulling any sub-optimal arms with mean reward gaps  greater than $\Delta_*$ is at most $\sum_{i \in [K]: \Delta_i \ge \Delta_*}  O \left( \frac{\ln \left(   T^{0.5(3-\alpha)} \Delta_i^2 \right) \ln^{\alpha}(T)}{\Delta_i} \right) + O \left( \frac{(3-\alpha) \ln \ln(T) \ln^{\alpha}(T)}{\Delta_i} \right)  \le  \sum_{i \in [K]: \Delta_i \ge \Delta_*}  O \left( \frac{\ln \left(   T  \right) \ln^{\alpha}(T)}{\Delta_*} \right) + O \left( \frac{\ln \ln(T) \ln^{\alpha}(T)}{\Delta_*} \right) 
               \le  O \left( \sqrt{KT \ln^{1+\alpha}(T)}\right)$.\end{proof}

\newpage

\begin{proof}[Proof of Lemma~\ref{lemma: theta}]

Let $\tau_s^{(i)}$ be the round by the end of which the empirical mean will be computed based on $2^{s}$ fresh observations.

We have
\begin{equation}
    \begin{array}{ll}
         & \sum\limits_{t=K+1}^{T} \mathbb{E} \left[\bm{1} \left\{i_t = i,  \overline{\mathcal{E}_i^{\theta}(t)}, n_i(t-1) \ge L_i \right\} \right] \\
       =   & \sum\limits_{t=K+1}^{T} \mathbb{E} \left[\bm{1} \left\{i_t = i,  \theta_i(t) > \hat{\mu}_{i,n_i(t-1)} + \sqrt{2\ln(T\Delta_i^2 \cdot \phi)} \cdot\sqrt{\frac{ \ln^{\alpha}(T)}{n_i(t-1)} } , n_i(t-1) \ge L_i \right\} \right] \\
          \le   & \sum\limits_{s=0}^{\log(T)} \mathbb{E} \left[ \sum\limits_{t =\tau_s^{(i)} +1 }^{\tau_{s+1}^{(i)}}\bm{1} \left\{i_t = i,  \theta_i(t) > \hat{\mu}_{i,n_i(t-1)} + \sqrt{2\ln(T\Delta_i^2 \cdot \phi)} \cdot\sqrt{\frac{ \ln^{\alpha}(T)}{n_i(t-1)} }  \right\} \right] \\
       \le & \sum\limits_{s=0}^{\log(T)}  2^{s+1}  \cdot \mathbb{P} \left\{ \text{MAX}_{i} > \hat{\mu}_{i,2^s} + \sqrt{2\ln(T\Delta_i^2 \cdot \phi)} \cdot\sqrt{\frac{ \ln^{\alpha}(T)}{2^s}  } \right\} \\
     \le &   \sum\limits_{s=0}^{\log(T)}  2^{s+1}  \cdot \phi \cdot  \frac{1}{2} e^{- \ln(T\Delta_i^2 \cdot \phi)} \\  
\le & O \left(T \cdot \phi \cdot  \frac{1}{T \Delta_i^2 \cdot \phi} \right) \\
\le & O \left(  \frac{1}{\Delta_i^2} \right) \quad,
    \end{array}
\end{equation}
which concludes the proof. %{\color{red}(Maybe add explanation for step (a).)}

\end{proof}

\begin{proof}[Proof of Lemma~\ref{lemma: hat}]
  From Hoeffding's inequality,    we have 
    \begin{equation}
        \begin{array}{ll}
           & \sum\limits_{t=K+1}^{T} \mathbb{E} \left[\bm{1} \left\{i_t = i, \overline{\mathcal{E}_i^{\mu}(t-1)}, n_i(t-1) \ge L_i \right\} \right] \\
     \le   &\sum\limits_{s=0}^{\log(T)} \mathbb{P} \left\{ \left| \hat{\mu}_{i,2^s} - \mu_i \right| \le \sqrt{\frac{\ln(T\Delta_i^2)}{2^s}} \right\} 2^{s+1} \\
     
     \le & \sum\limits_{s=0}^{\log(T)}  2e^{-2 \ln(T\Delta_i^2)} \cdot 2^{s+1} \\
     \le & O \left(T \cdot \frac{1}{T \Delta_i^2 \cdot T \Delta_i^2} \right) \\
     \le & O \left(\frac{1}{\Delta_i^2} \right)\quad,
        \end{array}
    \end{equation}
    which concludes the proof. 
\end{proof}

\begin{proof}[Proof of Lemma~\ref{martingale lemma}] %{\color{red}Done!!!}

For any $F_{t-1}$, we have
\begin{equation}
        \begin{array}{ll}
             & \mathbb{E} \left[   \bm{1} \left\{i_t = i, \theta_i(t) \le \mu_i + 0.5\Delta_i,\mathcal{T}_1(t)  \right\} \mid \mathcal{F}_{t-1} = F_{t-1}\right] \\
  \le &  \bm{1} \left\{\mathcal{T}_1(t)  \right\} \cdot \mathbb{E}\left[   \bm{1} \left\{\theta_1(t) \le \mu_i + 0.5\Delta_i, \theta_j(t) \le \mu_i + 0.5\Delta_i, \forall j \in [K] \setminus \{1\} \right\} \mid \mathcal{F}_{t-1} = F_{t-1} \right] \\
    = &  \bm{1} \left\{\mathcal{T}_1(t)   \right\} \cdot \mathbb{E} \left[   \bm{1} \left\{\theta_1(t) \le \mu_i + 0.5\Delta_i\right\}\mid \mathcal{F}_{t-1} = F_{t-1} \right] \cdot \mathbb{E} \left[\bm{1}  \left\{\theta_j(t) \le \mu_i + 0.5\Delta_i, \forall j \in [K] \setminus \{1\} \right\} \mid \mathcal{F}_{t-1} = F_{t-1} \right],  
        \end{array}
        \label{Sun 1}
    \end{equation}
         where the first inequality uses the fact that event $\mathcal{T}_1(t)$  is determined by the history information. Note if $h_1(t-1) \in [\phi]$, we have $\bm{1} \left\{\mathcal{T}_1(t)\right\} = 1$; if $h_1(t-1) = 0$, we have $\bm{1} \left\{\mathcal{T}_1(t)\right\} = 0$.

         We also have
             \begin{equation}
        \begin{array}{ll}
             &  \mathbb{E} \left[   \bm{1} \left\{i_t = 1, \theta_i(t) \le \mu_i + 0.5\Delta_i,\mathcal{T}_1(t) \right\} \mid \mathcal{F}_{t-1} = F_{t-1}\right] \\
       \ge &  \bm{1} \left\{\mathcal{T}_1(t)  \right\} \cdot \mathbb{E} \left[   \bm{1} \left\{\theta_1(t) >   \mu_i + 0.5\Delta_i \ge   \theta_j(t), \forall j \in [K]\setminus \{1\}\right\} \mid \mathcal{F}_{t-1} = F_{t-1} \right]\\
       = &  \bm{1} \left\{\mathcal{T}_1(t)\right\} \cdot \mathbb{E} \left[   \bm{1} \left\{\theta_1(t) >  \mu_i + 0.5\Delta_i\right\} \mid \mathcal{F}_{t-1} = F_{t-1} \right] \cdot \mathbb{E} \left[\bm{1} \left\{\theta_j(t) \le \mu_i + 0.5\Delta_i, \forall j \in [K] \setminus \{1\}\right\} \mid \mathcal{F}_{t-1} = F_{t-1} \right].
  \end{array}
        \label{Sun 2}
    \end{equation}

    Now,     we categorize all the possible $F_{t-1}$'s of $\mathcal{F}_{t-1}$ into two groups based on whether 
        $\bm{1} \left\{\mathcal{T}_1(t) \right\} = 0$ or 
    $\bm{1} \left\{\mathcal{T}_1(t)  \right\} = 1$.

\paragraph{Case 1:}   For any $F_{t-1}$ such that $\bm{1} \left\{\mathcal{T}_1(t)\right\} = 0$, 
    combining (\ref{Sun 1}) and (\ref{Sun 2}) gives
    \begin{equation}
        \begin{array}{ll}
           &   \mathbb{E} \left[   \bm{1} \left\{i_t = i, \theta_i(t) \le \mu_i + 0.5\Delta_i,\mathcal{T}_1(t)  \right\} \mid \mathcal{F}_{t-1} = F_{t-1}\right] \\
                     =   & 0 \\

%            =  &\mathbb{E} \left[   \bm{1} \left\{i_t = 1, \theta_i(t) \le \mu_i + 0.5\Delta_i,\mathcal{T}_1(t)  \right\} \mid \mathcal{F}_{t-1} = F_{t-1}\right]\\
          \le &\left(\frac{1 }{ \mathbb{P} \left\{\theta_{1,n_1(t-1)} > \mu_1 + 0.5\Delta_i \mid \mathcal{F}_{t-1} = F_{t-1} \right\} } -1\right) \cdot   \mathbb{E} \left[   \bm{1} \left\{i_t = 1  \right\} \mid \mathcal{F}_{t-1} = F_{t-1}\right]\quad,
        \end{array}
    \end{equation}
    where the last equality uses the fact that $0 < \left(\frac{1 }{ \mathbb{P} \left\{\theta_{1,n_1(t-1)} > \mu_i + 0.5\Delta_i \mid \mathcal{F}_{t-1} = F_{t-1} \right\} } -1\right) < +\infty$.% and is finite. 

\paragraph{Case 2:} 

%Recall $Z_{1,n_1(t-1)} \sim \mathcal{N} \left(\hat{\mu}_{1,n_1(t-1)}, \ \frac{\ln^{\alpha}(T)}{n_1(t-1)} \right)$. Let $Z'_{1,n_1(t-1)} \sim \mathcal{N} \left(\hat{\mu}_{1,n_1(t-1)}, \ \frac{\ln^{\alpha}(T)}{n_1(t-1)} + \frac{\ln(1.25/\delta_0)}{\varepsilon_0^2 \cdot n_1^2(t-1)}\right)$ and $Y  \sim \mathcal{N}\left(0, \ \frac{\ln(1/\delta_0)}{\varepsilon_0^2 \cdot n_1^2(t-1)}\right)$.

For any $F_{t-1}$ such that $\bm{1} \left\{\mathcal{T}_1(t) \right\} = 1$, we have
\begin{equation}
        \begin{array}{ll}
             & \mathbb{E} \left[   \bm{1} \left\{i_t = i, \theta_i(t) \le \mu_i + 0.5\Delta_i,\mathcal{T}_1(t)  \right\} \mid \mathcal{F}_{t-1} = F_{t-1}\right] \\
  \le &  \bm{1} \left\{\mathcal{T}_1(t)  \right\} \cdot \mathbb{E}\left[   \bm{1} \left\{\theta_1(t) \le \mu_i + 0.5\Delta_i, \theta_j(t) \le \mu_i + 0.5\Delta_i, \forall j \in [K] \setminus \{1\} \right\} \mid \mathcal{F}_{t-1} = F_{t-1} \right] \\
    = &  \bm{1} \left\{\mathcal{T}_1(t)   \right\} \cdot \mathbb{E} \left[   \bm{1} \left\{\theta_1(t) \le \mu_i + 0.5\Delta_i\right\}\mid \mathcal{F}_{t-1} = F_{t-1} \right] \cdot \mathbb{E} \left[\bm{1}  \left\{\theta_j(t) \le \mu_i + 0.5\Delta_i, \forall j \in [K] \setminus \{1\} \right\} \mid \mathcal{F}_{t-1} = F_{t-1} \right]\\
    = &  \mathbb{E} \left[   \bm{1} \left\{\theta_{1, n_1(t-1)} \le \mu_i + 0.5\Delta_i\right\}\mid \mathcal{F}_{t-1} = F_{t-1} \right] \cdot \mathbb{E} \left[\bm{1}  \left\{\theta_j(t) \le \mu_i + 0.5\Delta_i, \forall j \in [K] \setminus \{1\} \right\} \mid \mathcal{F}_{t-1} = F_{t-1} \right]\quad.
        \end{array}
        \label{Sun 11}
    \end{equation}
We also have 
             \begin{equation}
        \begin{array}{ll}
             &  \mathbb{E} \left[   \bm{1} \left\{i_t = 1, \theta_i(t) \le \mu_i + 0.5\Delta_i,\mathcal{T}_1(t) \right\} \mid \mathcal{F}_{t-1} = F_{t-1}\right] \\
       \ge &  \bm{1} \left\{\mathcal{T}_1(t)  \right\} \cdot \mathbb{E} \left[   \bm{1} \left\{\theta_1(t) >   \mu_i + 0.5\Delta_i \ge   \theta_j(t), \forall j \in [K]\setminus \{1\}\right\} \mid \mathcal{F}_{t-1} = F_{t-1} \right]\\
       = &  \bm{1} \left\{\mathcal{T}_1(t)\right\} \cdot \mathbb{E} \left[   \bm{1} \left\{\theta_1(t) >  \mu_i + 0.5\Delta_i\right\} \mid \mathcal{F}_{t-1} = F_{t-1} \right] \cdot \mathbb{E} \left[\bm{1} \left\{\theta_j(t) \le \mu_i + 0.5\Delta_i, \forall j \in [K] \setminus \{1\}\right\} \mid \mathcal{F}_{t-1} = F_{t-1} \right] \\
       = &  \underbrace{ \mathbb{E} \left[   \bm{1} \left\{\theta_{1, n_1(t-1)} >  \mu_i + 0.5\Delta_i\right\} \mid \mathcal{F}_{t-1} = F_{t-1} \right] }_{>0}\cdot \mathbb{E} \left[\bm{1} \left\{\theta_j(t) \le \mu_i + 0.5\Delta_i, \forall j \in [K] \setminus \{1\}\right\} \mid \mathcal{F}_{t-1} = F_{t-1} \right] \quad.
  \end{array}
        \label{Sun 22}
    \end{equation}
    From (\ref{Sun 11}) and (\ref{Sun 22}), we have
   \begin{equation}
        \begin{array}{ll}
           &   \mathbb{E} \left[   \bm{1} \left\{i_t = i, \theta_i(t) \le \mu_i + 0.5\Delta_i,\mathcal{T}_1(t)  \right\} \mid \mathcal{F}_{t-1} = F_{t-1}\right] \\
         \le   &\frac{\mathbb{P}  \left\{\theta_{1,n_1(t-1)} \le \mu_i + 0.5\Delta_i\mid \mathcal{F}_{t-1} = F_{t-1} \right\} }{\mathbb{P}  \left\{\theta_{1,n_1(t-1)} > \mu_i + 0.5\Delta_i\mid \mathcal{F}_{t-1} = F_{t-1} \right\}} \cdot \mathbb{E} \left[   \bm{1} \left\{i_t = 1, \theta_i(t) \le \mu_i + 0.5\Delta_i,\mathcal{T}_1(t)  \right\} \mid \mathcal{F}_{t-1} = F_{t-1}\right] \\
             \le  &\left(\frac{1 }{ \mathbb{P} \left\{\theta_{1,n_1(t-1)}> \mu_i + 0.5\Delta_i  \mid \mathcal{F}_{t-1} = F_{t-1} \right\}  } -1\right) \cdot   \mathbb{E} \left[   \bm{1} \left\{i_t = 1  \right\} \mid \mathcal{F}_{t-1} = F_{t-1}\right]\quad,
        \end{array}
    \end{equation}
    which concludes the proof.\end{proof}

\newpage

\section{Additional Experimental Results} \label{app: more exp}

\subsection{M-TS-Gaussian parameter selection}\label{appx:mtsg_pars}
Recall that in Section~\ref{sc:privacy}, we let $c = \frac{1}{2c_0 (b+1)} T^{0.5(1+\alpha)} \ln^{-1.5(1-\alpha)}(T)$ for any $b$ for M-TS-Gaussian to satisfy  $\sqrt{2c_0 T^{0.5(1-\alpha)}\ln^{1.5(1-\alpha)}(T)}$-GDP. To determine the best $b$ value for each $\alpha$ considered in Section~\ref{sc:privacy}, we conduct experiments with $b = \{0, 1, 500, 1000, 2000, 5000, 100000\}$. The results are shown in Figure~\ref {fig:main-fig_bernoulli_tsou}.

\begin{figure}[h]
    \centering
    \subfigure[Regret for  $O(\sqrt{T^{0.5} \ln^{1.5} T})$ -GDP  ($\alpha = 0$) . ]{
        \includegraphics[width=0.4\linewidth]{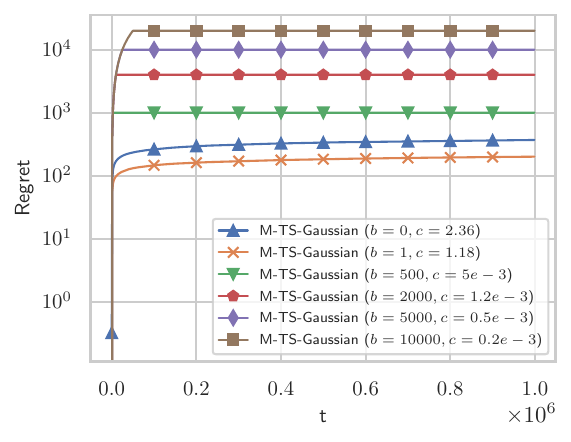}
        \label{fig:tregret_Bernoulli1}
    }
    \subfigure[Regret for $\sqrt{2c_0}$-GDP ($\alpha = 1$).]{
        \includegraphics[width=0.4\linewidth]{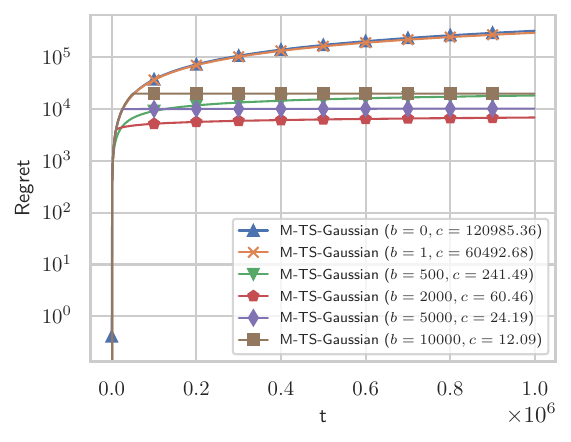}
        \label{fig:c0regret_Bernoulli2}
    }
    \caption{The regret of M-TS-Gaussian with parameters $b = \{0, 1, 500, 1000, 2000, 5000, 100000\}$ under two privacy guarantees.}
    \label{fig:main-fig_bernoulli_tsou}
\end{figure}

We can observe that when $\alpha = 0$, M-TS-Gaussian achieves the lowest regret with $b=1$ and $c=1.18$, as shown in Figure~\ref{fig:tregret_Bernoulli1}. When $\alpha = 1$, M-TS-Gaussian achieves the lowest regret with $b=2000$ and $c=60.46$, as shown in Figure~\ref{fig:c0regret_Bernoulli2}.

\subsection{Comparison with $(\epsilon, 0)$-DP algorithms}\label{app:epsi}
We compare DP-TS-UCB with $(\varepsilon, 0)$-DP algorithms with $\varepsilon = 0.5$: DP-SE~\citep{sajed2019optimal}, Anytime-Lazy-UCB \citep{hu2021near}  and Lazy-DP-TS~\citep{hu2022near}. These algorithms use the Laplace mechanism to inject noise.

\begin{figure}[h]
    \centering
    \subfigure[Regret for  $O(\sqrt{T^{0.5} \ln^{1.5} T})$ -GDP  ($\alpha = 0$) . ]{
        \includegraphics[width=0.4\linewidth]{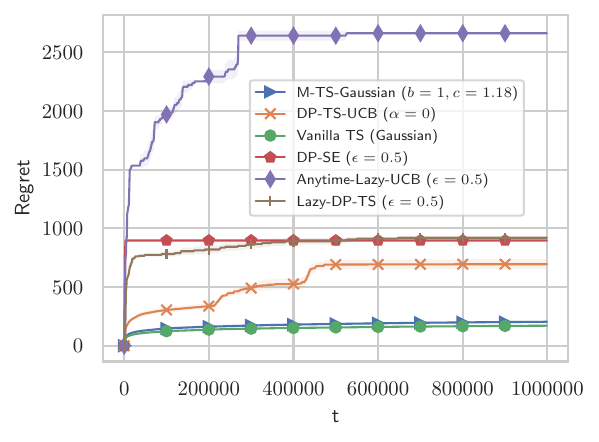}
        \label{fig:tregret_Bernoulli_epsi_e}
    }
    \subfigure[Regret for $\sqrt{2c_0}$-GDP ($\alpha = 1$).]{
        \includegraphics[width=0.4\linewidth]{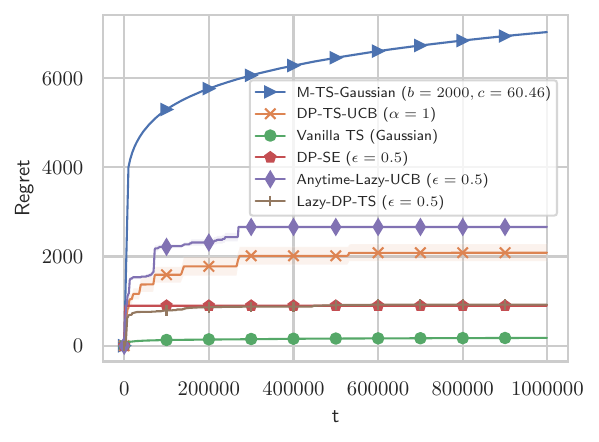}
        \label{fig:c0regret_Bernoulli_epsi_e}
    }
    \caption{The regret of DP-TS-UCB and M-TS-Gaussian under the same privacy guarantee with $\alpha = 0$ and $1$, with comparison to $(\epsilon,0)$-DP algorithms.}
    \label{fig:regret_same_privacy_epsilon}
\end{figure}

We can see that when $\alpha = 0$, both DP-TS-UCB and M-TS-Gaussian perform better than the $(\epsilon,0)$-DP algorithms, as shown in Figure~\ref{fig:tregret_Bernoulli_epsi_e}. When we increase $\alpha = 1$, M-TS-Gaussian performs worse than the $(\epsilon,0)$-DP algorithms, but DP-TS-UCB still outperforms Anytime-Lazy-UCB, as shown in Figure~\ref{fig:c0regret_Bernoulli_epsi_e}.

\end{document}